\documentclass[review]{elsarticle}

\usepackage{lineno,hyperref}
\modulolinenumbers[5]

\journal{arxiv}
\usepackage{url}
\usepackage{ marvosym }
\usepackage{graphicx}
\usepackage[table,usenames,dvipsnames]{xcolor}
\usepackage{array}
\usepackage{colortbl}
\usepackage{ amssymb }
\usepackage[small,bf]{caption}
\usepackage{float}
\usepackage{tabu}
\usepackage{colortbl}
\usepackage{ gensymb }
\usepackage{xfrac}
\usepackage{multirow}
\usepackage{amsmath}
\usepackage{fullpage}
\usepackage{pdflscape}
\usepackage{listings}
\usepackage{subcaption}
\usepackage{enumerate}
\usepackage{tabularx}
\usepackage{hyphenat}
\usepackage{algorithmic}
\usepackage{algorithm}
\usepackage{skak}
\usepackage{tikz}
\usepackage{pgfplots}
\usepackage{booktabs}
\newcolumntype{R}{>{\raggedleft\arraybackslash}X}
\newcolumntype{L}{>{\raggedright\arraybackslash}X}
\usepackage{amsthm} %Theorems
\newtheorem{theorem}{Theorem}
\newtheorem{ress}{Result}{\bf}{\it}
\newtheorem{corollary}{Corollary}
\newtheorem{lemma}{Lemma}

\DeclareMathOperator{\rank}{rank}
\DeclareMathOperator{\sign}{sign}
\usepackage{enumitem} %some itemize env w/o indent
\usepackage{adjustbox}
\usepackage{csquotes}
\MakeOuterQuote{"}
\usepackage[title]{appendix}
\usepackage{natbib}
\renewcommand{\cite}{\citep}
\begin{document}
\begin{frontmatter}
\title{A linear method for camera pair self-calibration and multi-view reconstruction with geometrically  verified  correspondences\tnoteref{fundingnote}
}
\tnotetext[fundingnote]{N. Melanitis is funded by General Secretariat for Research and Technology (GSRT) and Hellenic Foundation for Research and Innovation(HFRI)}
\author[1]{Nikos Melanitis\corref{cor1}\fnref{fn1}}
\ead{nimepar@gmail.com}
\author[1]{Petros Maragos\fnref{fn1}}
\ead{maragos@cs.ntua.gr}
\fntext[fn1]{ School of Electrical and Computer Engineering, National Technical University of Athens,  Athens, Greece}
\address[1]{National Technical University of Athens,9, Iroon Polytechniou Str.,
  15780 Zografos, Athens Greece}
\begin{abstract}
  We examine  3D reconstruction of architectural scenes in unordered sets of uncalibrated images. We introduce a linear method to  self-calibrate and find the metric reconstruction of a camera pair. We assume unknown and different focal lengths but otherwise known internal camera parameters and a known  projective reconstruction of the camera pair.  We recover two possible camera configurations in space and use the Cheirality condition, that all 3D scene points are in front of both cameras, to disambiguate the solution. We show in  two Theorems, first that the two solutions are in mirror positions and then the relations between their viewing directions. Our new method performs on par (median rotation error $\Delta R = 3.49^{\circ}$) with the standard approach of Kruppa equations ($\Delta R = 3.77^{\circ}$) for self-calibration and 5-Point algorithm for calibrated metric reconstruction of a camera pair. We reject erroneous image correspondences by introducing a method to examine whether point correspondences appear in the same order along $x, y$ image axes in image pairs. We evaluate this method by its precision and recall and show that it improves the robustness of point matches in architectural and general scenes. Finally, we integrate all the introduced methods to a 3D reconstruction pipeline. We utilize the numerous camera pair metric recontructions using rotation-averaging algorithms and a novel method to average focal length estimates.

\end{abstract}
\end{frontmatter}
\linenumbers
\section{Introduction}
Multi-view geometry (mvg) is a Computer Vision (CV) subfield that attempts to understand the structure of the 3D world given a collection of its images~\cite{hartley2003multiple}. As the binocular human vision is naturally 3D, the same underlying principles allow the recovery of the 3D world structure in mvg reconstruction methods. However, a prerequisite is to have calibrated cameras,  an assumption that is violated in unordered image sets. In this paper we focus on self-calibration and multi-view reconstruction using  relations between camera pairs.

Assuming a camera pair with unknown and different focal lengths as the only unknown internal parameters, a standard approach to self-calibration and metric reconstruction first applies the 7 point algorithm~\cite{hartley2003multiple} inside a RANSAC~\cite{ransacalgcit} procedure to find the fundamental matrix. In this projective framework, the Kruppa Equations~\cite{krupparef} are used to determine the unknown focal lengths. Next, applying the 5 point algorithm~\cite{5palgcit} inside a RANSAC procedure, leads to a metric reconstruction. Since focal lengths are recovered in a projective framework, only epipolar geometry constraints may be used to check the solution plausibility. Solving self-calibration and metric reconstruction problems  simultaneously permits the application of  the more intuitive and restrictive geometric arguments of the metric framework.

Self-calibration methods are derived from relations on the dual absolute conic (DAC) $Q^{*}_{\infty}$ and the dual image of the absolute conic (DIAC) $\omega_{\infty}^*$~\cite{176,195,krupparef}. However, existing methods require three or four images to provide a solution~\cite{176,195}, use numerical methods to determine DAC~\cite{176}, provide an initial DAC estimate that violates the rank-2 condition~\cite{176,195} and do not examine the relations between the recovered putative solutions~\cite{176}. In mvg reconstructions additional assumptions have been made to determine focal lengths, as availability of EXIF tags~\cite{sfmrot2, snavely2010bundler}, equality of focal lengths across all images~\cite{martinec2007robust, stewenius2005minimal} and vanishing points correspondences~\cite{sinha2010multi}.

Towards a multi-view reconstruction, camera pairs have been utilized.  Different estimates for a rotation matrix $R$ can be combined with a rotation  averaging algorithm~\cite{rotavealg2} and reconstructions of pairs of images can be combined with rotation registration methods~\cite{rotreg1, rotreg2, rotavealg1} to initialize an instance of Structure-from-Motion with known rotations~\cite{sfmrot1, sfmrot2} and produce a multiple-view reconstruction. 

Erroneous solutions in mvg problems are directly caused by erroneous or noisy image correspondences. Two complementary approaches, applying RANSAC procedures to repeateadly sample minimal point sets and verifying the initial point coresspondences, have been utilised to improve the validity of the recovered solution~\cite{chum2012homography, snavely2006photo}. 

In this paper, we derive a linear  method for the self-calibration and metric reconstruction of camera pairs with unknown and different focal lengths, unifying two problems that were previously solved independently, to a single system of equations. We further disambiguate the two solutions recovered by our method through the derivation of two theorems about the solutions' relations.  We improve the robustness and applicability of this method by introducing a procedure to verify tentative point correspondences between images,  using the Longest Common/Increasing Subsequence (LCS/LIS) problem~\cite{fredman1975computing}. The verification method is tailored for outdoors scenes of buildings, and is based on enforcing expected geometric properties of such scenes. We  integrate  our afforementioned methods to a multi-view reconstruction pipeline, utilizing  $L_{\infty}$-norm algorithms and introducing a method to average different estimates of a single focal length $f_i$, which uses the structure of the problem, specifically that each estimate for $f_i$ comes from a pair of images $i,j$ and is so paired with a second estimate $f_j$.

The rest of this article is organized as follows: Section~\ref{sec:RW} provides background on the reconstruction problem  and the verification of image correspondences.  Section~\ref{sec:Rec} introduces our method for self-calibration and metric reconstruction. Section~\ref{sec:GV} presents our method  for correspondences  verification between the images, first  Section~\ref{sec:GVlcs} presents a method which is reduced to the  LCS problem and then Section~\ref{sec:GVpract} presents the final practical algorithm. In Section~\ref{sec:AppReconstruction}, we integrate our methods to a reconstruction pipeline. In doing so, we develop novel averaging methods for $R, f$ estimates recovered from image pairs. Results for camera pair reconstruction, correspondences verification, focal length averaging and mv reconstruction are given in Section~\ref{sec:Res}.

\section{Background \& Related Work}
\label{sec:RW}
In the following  bold font (e.g. $\mathbf{v}$ ) is used for vectors and capital case normal font (e.g. $K$) is reserved for matrices.
\subsection{Elements of multiple view geometry}
\label{sec:RWmvg}
In this section we summarize  basic notions about the projection of 3D scenes to 2D planes~\cite{hartley2003multiple, faugeras2004geometry}. In a metric reconstruction  parallel world lines converge at the plane at infinity $\mathbf{\pi}_{\infty}$: $\begin{pmatrix} 0 & 0 & 0 & 1 \end{pmatrix}^T$. The absolute conic $\Omega_{\infty}$ is a conic on $\mathbf{\pi}_{\infty}$ which satisfies $X_1^2 + X_2^2 +X_3^2=0$, $X_4=0$, where $\mathbf{X}=\begin{pmatrix} X_1 & X_2 & X_3 & X_4 \end{pmatrix}^T$ is the homogeneous representation of world points.

By taking all the planes tangent to $\Omega_{\infty}$, we construct $Q_{\infty}^{*}$, which is the dual surface of $\Omega_{\infty}$. $Q_{\infty}^{*}$ is described in a metric reconstruction by the $4 \times 4$ matrix 
\begin{equation}
\label{eq:DACdef}
Q_{\infty}^{*}=
\begin{bmatrix}
I_{3 \times 3} &  \mathbf{0_3} \\
\mathbf{0^T_3} & 0  
\end{bmatrix} 
\end{equation}
Now, considering projective reconstructions of 3-space and projections to image plane we have the following Results~\cite{hartley2003multiple}:
\begin{ress}
\label{th:DIAC}
The projection of $Q^{*}$ by projection matrix $P$ in the image plane is the dual conic  
\begin{equation*}
C^{*}=PQ^{*}P^{T}
\end{equation*} 
\end{ress}
\begin{ress}
\label{th:pinfty}
If the 3-space is transformed by homography $H$, that is $\mathbf{X'}=H\mathbf{X}$, then planes of 3-space are transformed according to 
\begin{equation*}
\mathbf{\pi'}=H^{-T}\mathbf{\pi}
\end{equation*}
\end{ress}

\begin{ress}
\label{th:Fpractical1}
If $H$ is a $4\times 4$ matrix representing a projective transformation of 3-space, then the fundamental matrices corresponding to the pairs of camera matrices $\{P, P'\}$, $\{PH, P'H\}$ are the same.
\end{ress} 
\begin{ress}
\label{th:Fpractical2}
Suppose the $\rank 2$ matrix $F$ can be decomposed in two different ways as
\begin{align*}
F&= {[\mathbf{a}]}_xA \\
F&= {[\hat{\mathbf{a}}]}_x\hat{A} 
\end{align*}
then 
\begin{align*}
\hat{\mathbf{a}}&=\kappa \mathbf{a} \\
\hat{A}&=\kappa^{-1} (A+\mathbf{av^T}) 
\end{align*}
for some non-zero constant $\kappa$ and 3-vector $\mathbf{v}$
\end{ress}
Using the preceding Results, we formulate the equations to solve the camera self-calibration problem and to determine $\mathbf{\pi_{\infty}}$ position in a projective reconstruction.
\subsection{Notation}
\label{sec:RWnotation}
We summarize our notation in Table~\ref{tab:notation}.
\begingroup
\hyphenpenalty 10000
\exhyphenpenalty 10000
\begin{table}[tpb]
\scriptsize
\begin{center}
\begin{tabularx}{0.7\linewidth}{X X }
\midrule[1.5pt]
 \bfseries{subscripts} & \\ \midrule[0.25pt]
 M & Metric Reconstruction, e.g. $P_M$ \\ \midrule[0.25pt] 
 P & Projective Reconstruction, e.g. $P_P$ \\ \midrule[0.25pt]
 1(or 2) & Refers to Camera 1(or 2) in camera pair, e.g $P_{P1}$ \\ \midrule[0.25pt]
 GT & Ground Truth \\ \midrule[1.5pt]
\bfseries{superscripts} & \\ \midrule[0.25pt]
1(or 2)& Refers to solution 1(or 2) for the second camera, e.g $\mathbf{x^1}$\\ \midrule[0.25pt]
 Accents, as in $\mathbf{p'}$ and  $\mathbf{p}$ & Discriminate between the 2 solutions for camera 2 \\ \midrule[1.5pt]
\bfseries{P matrix representations}& \\ \midrule[0.25pt]
 $\begin{bmatrix} KR & -KR\mathbf{C}\end{bmatrix}$ & Metric reconstruction  \\ \midrule[0.25pt]
 $\begin{bmatrix}P_i & \mathbf{a}  \end{bmatrix}$ & Metric reconstruction \\ \midrule[0.25pt]
 $\begin{bmatrix}KR &\mathbf{a}  \end{bmatrix}$ & Metric reconstruction \\ \midrule[0.25pt]
 $\mathbf{m_i}$& i-row vector of left $3\times 3$ P matrix block \\ \midrule[0.25pt]
 $\begin{bmatrix} {[\mathbf{a}]}_xF& \mathbf{a} \end{bmatrix}, \mathbf{a}:F^T\mathbf{a}=\mathbf{0}$& Projective Reconstruction in canonical form  \\ \midrule[1.5pt]
\bfseries{Simplifications} &  \\ \midrule[0.25pt]
$\mathbf{C^i}\triangleq \mathbf{C_i}$ & Appear in some Lemmas \\ \midrule[0.25pt]
 $P^i \equiv P_i$ &  \\ \midrule[0.25pt]
 $\omega^{*i} \triangleq \omega^*_i$ &
\\
\lasthline
\end{tabularx}
\caption{A summary of notation, with references to uses in the text}
\label{tab:notation}
\end{center}
\end{table}
\endgroup

\subsection{Verifying point correspondences}
\label{sec:RWgeover}

In standard approaches to find image correspondences  regions of interest are described by local feature descriptors~\cite{lowe2004distinctive}. Consequently,
erroneous  matches occur between  similar in local appearance image regions. 

A way to reject erroneous matches is  using  arguments about the geometry of the depicted scenes (geometric verification).  In SIFT features, Hough transform was used to acquire the orientation of the detected features~\cite{lowe2004distinctive}. Another common approach  applies a rudimentary transform (e.g affine, similarity) between the images,  to reject some correspondences before fitting the full model~\cite{turcot2009better,philbin2007object,perd2009efficient,chum2004enhancing}. For these methods we mention that:
\begin{itemize}
\item Most  require specific image features to be extracted, from which special parameters are used to fit  the  transform 
\item Result  in rejection of a large number of correspondences
\item When we tried using a similarity transform to  geometrically verify matches  in the reconstruction problem, our results did not improve
\end{itemize}

Another direction is to improve the covariance of local feature descriptors~\cite{perd2009efficient, chum2012homography, chum2006geometric}.  The regions of interest can be first transformed before extracting a feature descriptor~\cite{perd2009efficient} or ellipses may be matched instead of points~\cite{chum2012homography}. 

Finally, the neighborhoods of putative matched points are examined in some verification methods. Such approaches include counting the number of correspondences between the  neighborhoods of two tentative point matches~\cite{sivic2003video} or  examining the order of matched features between the neighborhoods and counting  the number of features out of order~\cite{wu2009bundling}. The geometric verification method we propose uses properties concerning the order of matched points as well.
%(visual proof lemma 11: \input{graphRotation.tex}--->oxi se auto to section omos)
\subsection{Approaches to multiple view reconstruction}
In a reconstruction pipeline, initially  Structure from Motion (SfM) is solved to get   $P,\quad \mathbf{X}$  assuming  image point correspondences  and self-calibrated cameras. The fundamental method to solve SfM  is Bundle Adjustment (BA)~\cite{BAalgcit}, an iterative, numerical algorithm  to minimize the reprojection error of the recovered solution.

In standard approaches to SfM a sequence of SfM sub-problems are solved (sequential SfM)~\cite{snavely2010bundler, snavely2006photo, wu2011multicore}.  In each iteration, more, possibly uncalibrated, cameras and world points are added to the SfM problem which is  solved using BA. However such methods are sensitive to the initial camera pair selection, solve a large number of optimization problems numerically and optimize an objective function with possibly multiple local minima.

A different approach has been developed for solving the SfM with known Rotations problem within the framework of optimal algorithms in multiple-view geometry (mvg) and $L_{\infty}$ mvg algorithms~\cite{dalalyan2009l_1, hartley2007optimal, sfmrot1, sfmrot2, olsson2010generalized, zach2010practical}. In this formulation, the camera rotation matrices $R$ are given.  SfM is formulated as a convex-optimization problem, for which a unique global minimum exists. For the actual solution of SfM with known rotations, either a sequence of Second-order cone programs are solved to arrive at an exact solution, or approximate solutions are recovered by solving SOCP or Linear programs~\cite{martinec2007robust, enqvist2011non, sfmrot2, sinha2010multi}. BA  may still be applied as a last fine-tuning of the solution.

A SfM solution, allows the reconstruction of a low number of 3D points (sparse point cloud), limited by the number of image points correspondences. \emph{Multi-view stereo} (mvs) algorithms can be used at this point to produce a \emph{dense point cloud}, which contains a much larger number of 3D points~\cite{furukawa2010accurate}. Finally, surface reconstruction algorithms can be used to produce a 3D surface~\cite{kazhdan2006poisson}.

\section{A method for Metric Reconstruction in pairs of Uncalibrated Images}
\label{sec:Rec}
\subsection{Formulation of System Equations}
\label{sec:RecForm}
Let us consider two cameras $P_1,P_2$ and further that  $P_1$ coordinate system is aligned with the world coordinate system. Let us further assume, that the corresponding image coordinate systems are selected so that the internal parameters of each camera $K_i$ can be written as
 \begin{equation*}
 K_i=\begin{bmatrix} f_i & 0 & 0 \\ 0&f_i &0 \\ 0& 0& 1 \end{bmatrix}
\end{equation*}
where $f_i$ is the focal length. The previous assumptions are routinely employed in multiple view geometry and are thoroughly discussed in the literature~\cite{hartley2003multiple}.\\
We start from a projective reconstruction of the 2 cameras, given by $P_{P1},P_{P2}$, which is related to the metric reconstruction by a world (3D) homography $H$ as in 
\begin{equation}
\label{eq:PH}
\begin{aligned}
P_{M1}&=P_{P1}H \\
P_{M2}&=P_{P2}H
\end{aligned}
\end{equation}
Using Result~\ref{th:DIAC}, Eq.\eqref{eq:DACdef}, we project $Q_{\infty}^{*}$ to the image plane of camera 2. For this projection, $
\omega_2^{*}$ we have:
\begin{equation}
\label{eq:omega}
 \begin{bmatrix} f_2^2 & 0 & 0 \\ 0&f_2^2 &0 \\ 0& 0& 1 \end{bmatrix}=\omega_2^{*}=P_{P2}HQ_{\infty}^{*}H^TP_{P2}^T
\end{equation}
To introduce the unknowns in Eq.~\eqref{eq:omega}, we use the canonical representation of the projective reconstruction, so that $ P_{P1}= \begin{bmatrix} I & \mathbf{0}\end{bmatrix}$. From Eq.~\eqref{eq:PH}, we have for the homography 
\begin{equation*}
  H =\begin{bmatrix} K_1 & \mathbf{0} \\ \mathbf{v^T}& \sigma \end{bmatrix}
  \end{equation*}
where $\mathbf{v}$ is yet undetermined and the scale factor $\sigma$ can be ignored ($\sigma = 1$).\\
To fully determine $H$, we turn to the plane at infinity\begin{equation*}
\mathbf{\pi_{\infty,P}} \triangleq \begin{pmatrix} \mathbf{p} & 1\end{pmatrix}^T \triangleq \begin{pmatrix} p_1 & p_2 & p_3 & 1\end{pmatrix}^T
\end{equation*}
Using Result~\ref{th:pinfty} we arrive at 
\begin{equation}
\label{eq:Hform}
H = \begin{bmatrix} K_1 & \mathbf{0} \\ -\mathbf{p^T}K_1 & 1 \end{bmatrix}
\end{equation}
Substituting $H$ from Eq.~\eqref{eq:Hform} to Eq.~\eqref{eq:omega} we get \begin{equation}
\label{eq:omInter}
\omega_{2}^{*}=P_{P2}
\begin{bmatrix} K_1K_1^T & -K_1K_1^T\mathbf{p} \\ -\mathbf{p^T}K_1K_1^T & \mathbf{p^T}K_1K_1^T\mathbf{p}
 \end{bmatrix}P_{P2}^T
\end{equation}
Eq.~\eqref{eq:omInter} comprise a non-linear system with respect to the five unknowns (plane at infinity coordinates and focal lengths) we pursue to determine to acquire a metric reconstruction of the scene. We note that  $\omega^{*}_{\infty}$ is symmetric by definition, and is also homogeneous, thus it provides five independent equations.
\subsection{Linearization}
\label{sec:RecLin}
In Eq.~\eqref{eq:omInter}, we substitute
\begin{equation*}
P_{P2} \triangleq \begin{bmatrix}
p_{11} & p_{12} & p_{13} & p_{14} \\
p_{21} & p_{22} & p_{23} & p_{24} \\
p_{31} & p_{32} & p_{33} & p_{34}
\end{bmatrix}
\end{equation*}
We group the unknowns in the following complexes 
\begin{equation}
\label{eq:Xteam}
\mathbf{x_o} \triangleq\begin{pmatrix} f_1^2 \\ f_2^2 \\ f_1^2p_1^2+f_1^2p_2^2+p_3^2\\p_3\\ f_1^2p_1\\f_1^2p_2 \end{pmatrix}
\end{equation} 
The augmented matrix $ [A|\mathbf{b}] $ for the linear  system 
\begin{equation}
\label{eq:Axob}
A\mathbf{x_o}=\mathbf{b}
\end{equation}
is then given by 
\begin{equation}
\label{eq:homg}
\resizebox{0.9\linewidth}{!}{\ensuremath{
\left(\begin{array}{rrrrrrr}
p_{21}^{2} + p_{22}^{2} & -1 & p_{24}^{2} & -2 \, p_{23} p_{24} & -2 \, p_{21} p_{24} & -2 \, p_{22} p_{24} & -p_{23}^{2} \\
p_{21} p_{31} + p_{22} p_{32} & 0 & p_{24} p_{34} & -p_{23} p_{34} - p_{24} p_{33} & -p_{21} p_{34} - p_{24} p_{31} & -p_{22} p_{34} - p_{24} p_{32} & -p_{23} p_{33} \\
p_{11} p_{31} + p_{12} p_{32} & 0 & p_{14} p_{34} & -p_{13} p_{34} - p_{14} p_{33} & -p_{11} p_{34} - p_{14} p_{31} & -p_{12} p_{34} - p_{14} p_{32} & -p_{13} p_{33} \\
p_{11}^{2} + p_{12}^{2} & -1 & p_{14}^{2} & -2 \, p_{13} p_{14} & -2 \, p_{11} p_{14} & -2 \, p_{12} p_{14} & -p_{13}^{2} \\
p_{11} p_{21} + p_{12} p_{22} & 0 & p_{14} p_{24} & -p_{13} p_{24} - p_{14} p_{23} & -p_{11} p_{24} - p_{14} p_{21} & -p_{12} p_{24} - p_{14} p_{22} & -p_{13} p_{23} \\
p_{31}^{2} + p_{32}^{2} & 0 & p_{34}^{2} & -2 \, p_{33} p_{34} & -2 \, p_{31} p_{34} & -2 \, p_{32} p_{34} & -p_{33}^{2} + 1 
\end{array}\right)}}
\end{equation}
We derived  the above equations (in order of appearance) from elements $\omega_2^{*}(2,2)$, $\omega_2^{*}(2,3)$,  $\omega_2^{*}(1,3)$, $\omega_2^{*}(1,1)$, $\omega_2^{*}(1,2)$, $\omega_2^{*}(3,3)$ of $\omega_2^{*}$. In the following, we use the first five equations as explained in Section~\ref{sec:RecCon}.

The matrix of Eq.~\eqref{eq:homg} is rank deficient. Thus, we presented a linear system of five (in the best case) linearly-independent equations, in six unknowns. To solve it, we turn to the polynomial relations between the coordinates of  $\mathbf{x_o}$.
\subsection{Recovering the solutions}
\label{sec:RecSol}
Taking five of Eqs.~\eqref{eq:Axob}  we have the linear system
\begin{equation}
\label{eq:A5xob}
A_5\mathbf{x_o}=\mathbf{b_5}
\end{equation}
Applying Gaussian elimination to~\eqref{eq:A5xob}, we bring the augmented matrix to the form
\begin{equation}
\label{eq:rref}
\begin{bmatrix}
1 & 0 & 0 & 0 &0 &{\color{ForestGreen}0}& b_1\\
0 & 1 & 0 & 0 &0 &{\color{ForestGreen}0} &b_2\\
0 & 0 & 1 & 0 &0 &{\color{NavyBlue}0} &b_3\\
0 & 0 & 0 & 1 &0 &c &b_4\\
0 & 0 & 0 & 0 &1 &d & b_5 
\end{bmatrix}
\end{equation} 
where:
\begin{enumerate}
\item The elements in default font, are in the usual form expected when we apply  Gaussian elimination in the general case
\item The elements in {\color{ForestGreen} green} font, are a result of the problem's structure, that is of the special relations in Eq.~\eqref{eq:rref}
\item Finally, the element in {\color{NavyBlue} blue} font, is as given when we use the canonical representation for the projective reconstruction, which is:
\begin{equation} P_{P1}= \begin{bmatrix}I & \mathbf{0}\end{bmatrix}, P_{P2}=\begin{bmatrix} {[ \mathbf{a}]}_x F &\mathbf{a}\end{bmatrix}   \end{equation} 
Where  $\mathbf{a}$ is the left null vector of $F$,
$ F^T\mathbf{a}=\mathbf{0} $. By using the canonical pair, the leftmost $3\times 3$ block in $P_{P2}$ is rank $2$, and consequently has linearly-dependent row-vectors
\end{enumerate}
The derivation of Eq.~\eqref{eq:rref} is given in the Supplementary Material.

To solve for the focal lengths ($f_1,f_2$) and $\mathbf{\pi_{\infty}}$ ($p_1,p_2,p_3$), we now have from \eqref{eq:rref}
\begin{align}
\label{eq:rref1}
f_1^2&=b_1\\
\label{eq:rref2}
f_2^2&=b_2\\
\label{eq:rref3}
f_1^2p_1^2+f_1^2p_2^2+p_3^2&=b_3\\
\label{eq:rref4}
p_3+cf_1^2p_2&=b_4\\
\label{eq:rref5}
f_1^2p_1+df_1^2p_1&=b_5
\end{align}
We substitute $p_1$, from~\eqref{eq:rref5}, and $p_3$, from~\eqref{eq:rref4}, to Eq.~\eqref{eq:rref3}, and obtain a second-order equation with respect to $p_2$. Thus, we determine $f_1,f_2$ uniquely and $p_1,p_2,p_3$ with a two-way ambiguity. We refer to those two solutions as

\begin{equation}
\label{eq:sols}
\begin{aligned}
\mathbf{x_o^1}&= \begin{pmatrix}
 b_1 & b_2 & b_3 & p_3 & f_1^2p_1 &f_1^2p_2 \end{pmatrix}^T \\
\mathbf{x_o^2}&= \begin{pmatrix}
 b_1 & b_2 & b_3 & p_3' & f_1^2p_1' &f_1^2p_2' \end{pmatrix}^T
\end{aligned}
\end{equation}
\subsection{The effect of homogeneous representation on the derived equations}
\label{sec:RecCon}
 In homogeneous  coordinate systems representations equal up to a multiplicative constant refer to the same entity. We explore here how this ambiguity affects the formulation of Eq.~\eqref{eq:Axob}. 

Let 
\begin{itemize}
\item $\{P_{GT1}',P_{GT2}' \} \triangleq \{ \begin{bmatrix} I &\mathbf{0}\end{bmatrix} ,\begin{bmatrix}A & \mathbf{a} \end{bmatrix} \}$, be the ground truth camera matrices we aim to recover
\item  $\{P_{P1},P_{P2} \} \triangleq \{ \begin{bmatrix} I &\mathbf{0}\end{bmatrix} ,\begin{bmatrix}\hat{A} & \hat{\mathbf{a}} \end{bmatrix} \}$, be the starting projective reconstruction
\end{itemize}
$P_{GT1}'$ is related to $P_{GT1}$ by the  homography
\begin{equation*}
H_k = \begin{bmatrix}
 K_1^{-1} & \mathbf{0} \\ \mathbf{0^T} & 1
\end{bmatrix}
\end{equation*}
Thus, we get from $P_P$ to $P_{GT}$ by homography $H$, from $P_P$ to $P_{GT}'$ by $H'$ and from $P_{GT}$ to $P_{GT}'$ by $H_k$.

For the camera pairs, we have the Fundamental matrices
\begin{align*}
F_{P}&= \begin{bmatrix}{[\mathbf{a}]}_xA  \end{bmatrix} \\
F_{GT}'&= \begin{bmatrix}{[\hat{\mathbf{a}}]}_x\hat{A}  \end{bmatrix}
\end{align*}
From Result~\ref{th:Fpractical1}, since  $P_{GT}', P_{P}$ reconstructions are related by $H'$, the reconstructions share a common Fundamental matrix. Since Fundamental matrices are homogeneous entities, we have
\begin{equation*}
F_{P}=\epsilon F_{GT}'
\end{equation*}
Now, we turn to Result~\ref{th:Fpractical1} and get
\begin{align*}
\hat{\mathbf{a}}&=\kappa \mathbf{a} \\
\hat{A}&=\epsilon^{-1} \kappa^{-1} (A+\mathbf{av^T})
\end{align*}
We write the previous equations in matrix form to get the projective transformation $H'$
\begin{equation*}
H'\triangleq \begin{bmatrix} \kappa^{-1}\epsilon^{-1}I & \mathbf{0} \\ \kappa^{-1}\epsilon^{-1}\mathbf{v^T} & \kappa  \end{bmatrix}
\end{equation*}
$H'$ satisfies
\begin{equation}
\label{eq:Fpractical4}
\begin{aligned}
\kappa^{-1}\epsilon^{-1} P_{GT1}'&=P_{P1}H'\\
P_{GT2}'&=P_{P1}H'
\end{aligned}
\end{equation}
Now, we get $H$ from $H', H_k^{-1}$
\begin{equation*}
\begin{bmatrix}
\kappa^{-1}\epsilon^{-1}K_1 & \mathbf{0} \\ \kappa^{-1}\epsilon^{-1}\mathbf{v^T}K_1 & \kappa
\end{bmatrix}
\end{equation*}
We set the bottom-right element to $1$, as we disregard the true scale of the reconstruction, and get the final form of $H$
\begin{equation}
\label{eq:Fpractical1}
\begin{bmatrix}
\kappa^{-1}\epsilon^{-1}K_1 & \mathbf{0} \\ \kappa^{-1}\epsilon^{-1}\mathbf{v^T}K_1 & 1
\end{bmatrix}
\end{equation}
One should compare the homographies of Eq.~\eqref{eq:Fpractical1} and Eq.~\eqref{eq:Hform}. Using Eq.~\eqref{eq:Hform} instead of Eq.~\eqref{eq:Fpractical1}, we get from $P_P$ to 
\begin{align*}
P_{M1}&=\begin{bmatrix} K_1 & \mathbf{0} \end{bmatrix}\\
P_{M2}&=\begin{bmatrix}\mu K_2R_2 &\mathbf{a} \end{bmatrix}
\end{align*}
We observe that the translation direction $\mathbf{a}$ is correct but the left-most $3\times 3$ block of camera $2$ is multiplied by a constant $\mu$. 

To see how the constants in Eq.~\eqref{eq:Fpractical1} affect Eq.~\eqref{eq:Axob}, we substitute Eq.~\eqref{eq:Fpractical1} in Eq.~\eqref{eq:omega} and get for $\omega_{2}^*$ the expression 
\begin{equation*}
\resizebox{0.9\linewidth}{!}{\ensuremath{
\omega_{2}^{*}=P_{P2}
\begin{bmatrix}(\kappa \epsilon)^{-2} K_1K_1^T & -(\kappa \epsilon)^{-2}K_1K_1^T\mathbf{p} \\ -(\kappa \epsilon)^{-2}\mathbf{p^T}K_1K_1^T & (\kappa \epsilon)^{-2}\mathbf{p^T}K_1K_1^T\mathbf{p}
 \end{bmatrix}P_{P2}^T}}
\end{equation*}
To avoid the determination of additional unknowns in Eq.~\eqref{eq:Axob}, we have 
\begin{itemize}
\item All equations derived from $\omega_2^*$ elements off the diagonal are   of the form $\mathbf{ax_o = 0}$, thus the constant $(\kappa \epsilon)^{-1}$ can be eliminated 
\item The equation derived from element $\omega_2^*(3,3)$ cannot be used without determining additional constants. So, we may only use the rest five of the six original equations of~\eqref{eq:Axob}  
\end{itemize}
The complete method to solve the metric reconstruction and self calibration problem follows:
\begin{enumerate}
\item We solve the system~\eqref{eq:Axob}, keeping five equations and discarding the equation derived from $\omega_2^*(3,3)$
\item In the previous step (1), we recovered $c^{-1}f_2^2$. To fully determine $f_2$, many different approaches are possible. We propose to repeat step 1, putting camera 2 at the origin of the coordinate system (in  place of camera 1). This can be done by transposing the Fundamental matrix $F$ for the camera pair. Following this approach, we may additionally determine the constant $\kappa \epsilon$
\item Using the homography of Eq.~\eqref{eq:Hform}  or Eq.~\eqref{eq:Fpractical1}, we recover the metric reconstruction $P_{M1},P_{M2}$. Depending on the homography used, one camera matrix ($P_{M2}$ for Eq.~\eqref{eq:Hform}  or $P_{M1}$ for Eq.~\eqref{eq:Fpractical1}) will have the left-most $3\times 3$ block multiplied by a constant. This has no effect on the correctness of the representation, and the image points are the same in each case
\end{enumerate}
\subsection{Solution disambiguation and geometric relations of the two solutions}
\label{sec:RecGeo}
\begin{figure}[tbp]
\centering
\captionsetup{singlelinecheck=off}
\includegraphics[width=\linewidth]{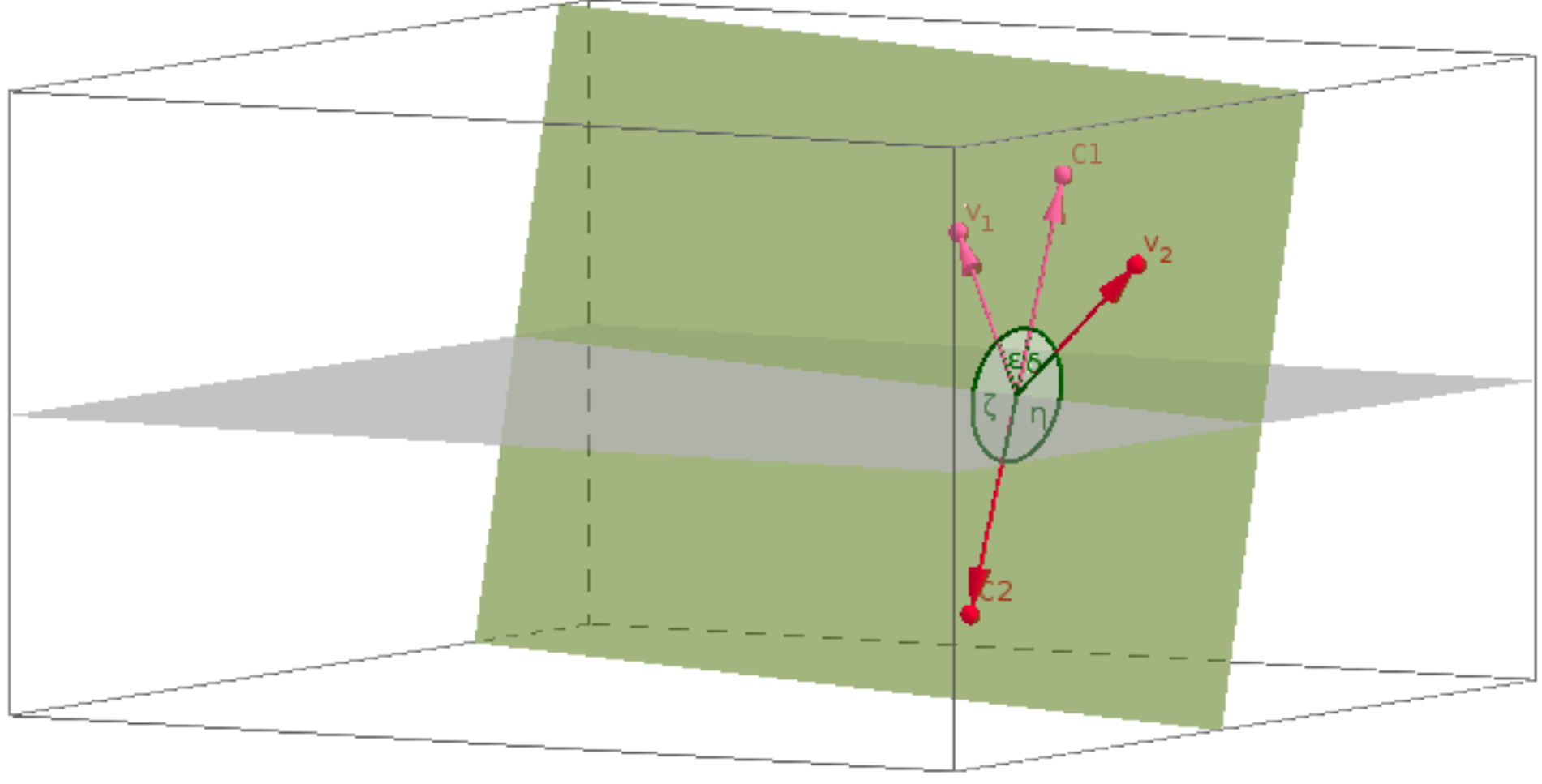}
\caption{We visualise the geometric relations of Theorems~\ref{th:mirror},\ref{th:viewdir}. In the graph, we display the centers of projection~($\mathbf{C^1,C^2}$) and viewing directions~($\mathbf{v_1,v_2}$), for each of the two solutions of Eq.~\ref{eq:sols}. In pink, we display solution $1$ and in red solution $2$. The common plane of $\mathbf{C_1,C_2,v_1,v_2}$ is highlighted %and we indicate the following angles on this plane 
%\begin{align*}\phase{ \mathbf{v_1,C_1}} \triangleq \epsilon &= \delta \triangleq  \phase{ \mathbf{v_2,C_1}}\\
 %\phase{ \mathbf{v_1,C_2}} \triangleq \zeta &= \eta \triangleq  \phase{ \mathbf{v_2,C_2}}\end{align*}
 } 
\end{figure}

We use the Cheirality condition (Corollary~\ref{cor:cheirality}) to determine the valid  solution of Eq.~\eqref{eq:Axob}. Whenever the two recovered solutions represent cameras with divergent viewing directions, Cheirality condition is more likely to identify the valid solution. We explore in Theorems~\ref{th:mirror} and~\ref{th:viewdir} the geometric relations between the two solutions, aiming to vizualize solutions' relations and disambiguation. Proofs of Theorems~\ref{th:mirror} and~\ref{th:viewdir} are outlined in Figs.~\ref{fig:Outline1} and~\ref{fig:Outline2}. Full proofs are given in the Supplemental Material.

\tikzset{
 header node/.style={
    Minimum Width=header nodes,
    font=\strut\tiny\ttfamily,
    text depth=+0pt,
    fill=white, draw},
  header/.style={%
    inner ysep=+1.5em,
    append after command={
      \pgfextra{\let\TikZlastnode\tikzlastnode}
      node [header node] (header-\TikZlastnode) at (\TikZlastnode.north) {#1}
      node [span=(\TikZlastnode)(header-\TikZlastnode)] at (fit bounding box) (h-\TikZlastnode) {}
    }
  },
  every path/.style={
               thick
    }
}
 
%preamble needed for figures that appear later in text...
\begin{figure*}
\centering
\includegraphics[width=\textwidth]{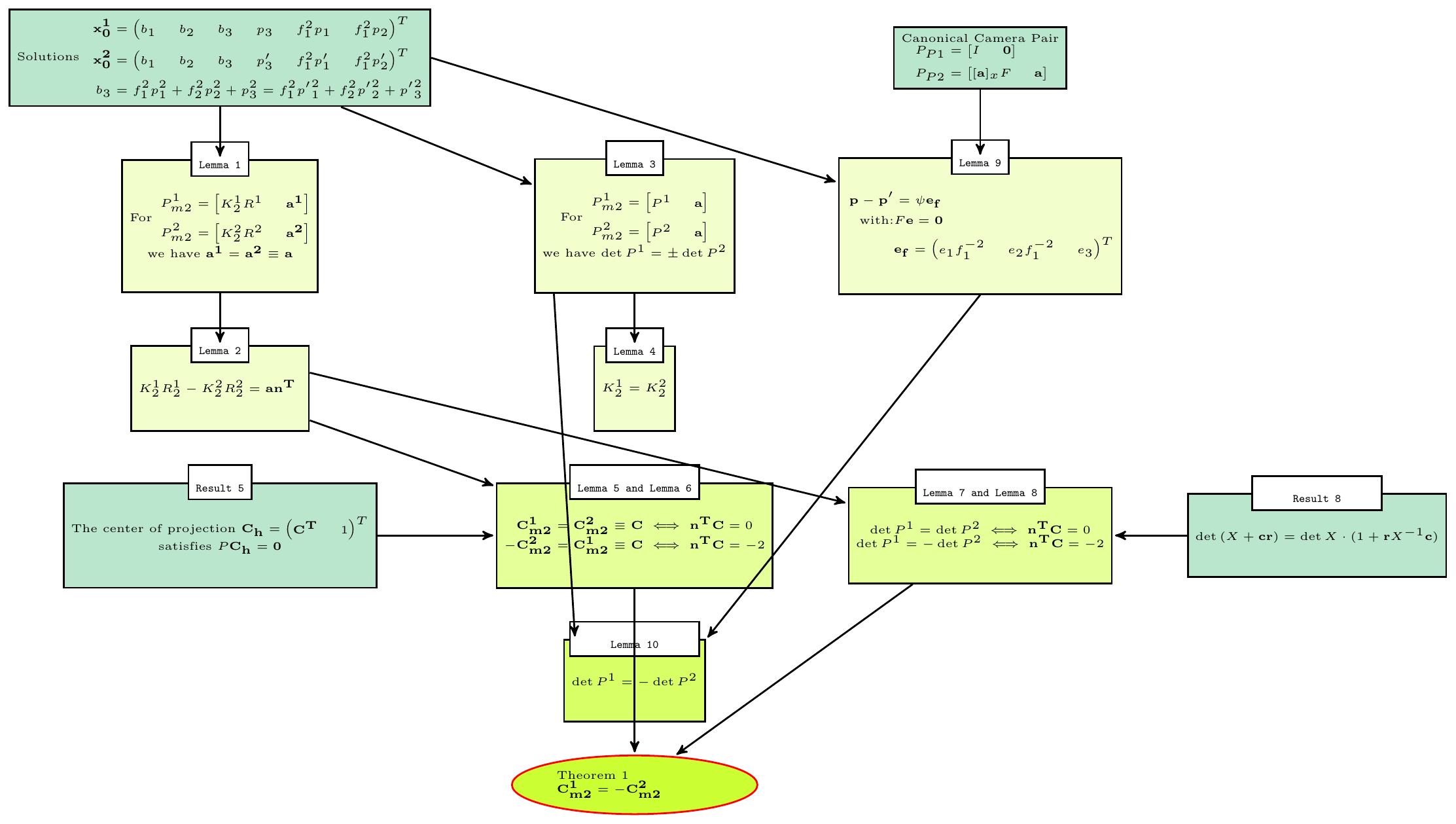}
\caption{An illustration of the intermediate results  leading to Theorem 1}
\label{fig:Outline1}
\end{figure*}

\begin{figure*}
\centering
\includegraphics[width=\textwidth]{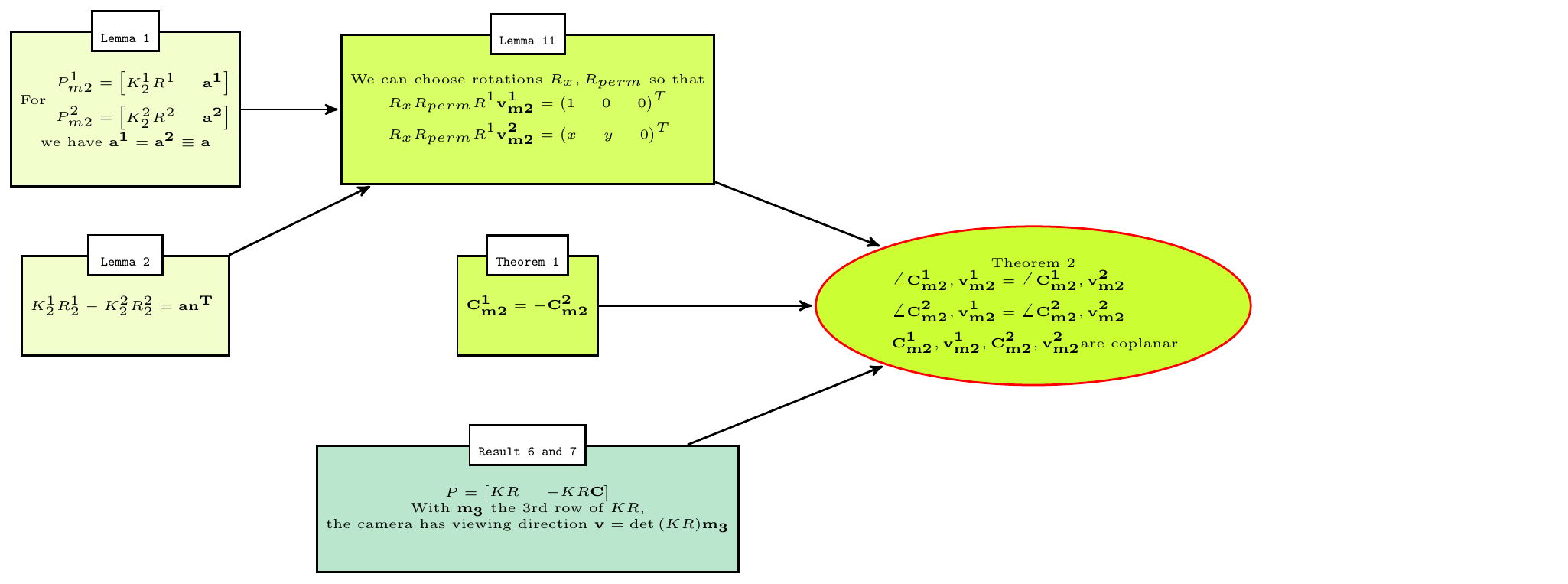}
\caption{An illustration of the intermediate results  leading to Theorem 2}
\label{fig:Outline2}
\end{figure*}

\begin{theorem}
\label{th:mirror}
Let
\begin{equation*}
 \{P_{m1}^1,P_{m2}^1\},\{P_{m1}^2,P_{m2}^2\}
\end{equation*}
 denote the reconstructions derived from~\eqref{eq:sols}. Then, cameras $P_{m2}^1,P_{m2}^2$ are in mirror positions with respect to the origin (position of $P_{m1}^{1},P_{m1}^{2}$). The centers of projection $\mathbf{C_{m2}^{1},C_{m2}^{2}}$ satisfy 
\begin{equation*}
%\label{eq:mirrorC}
\mathbf{C_{m2}^{1}=-C_{m2}^{2}}
\end{equation*}
\end{theorem}
\begin{theorem}
\label{th:viewdir}
 Let camera $1$ be positioned on the origin of the world coordinate system, with a viewing direction aligned to    $z$ axis. We denote  $\mathbf{v_{m2}^1,v_{m2}^2}$  the viewing directions of  $P_{m2}^1,P_{m2}^2$ and  $\mathbf{C_{m2}^1,C_{m2}^2}$  the position vectors of the corresponding centers of projection. Then, $\mathbf{C_{m2}^1,C_{m2}^2}$ bisect the angles formed by   $\mathbf{v_{m2}^1,v_{m2}^2}$, in the plane defined by $\mathbf{v_{m2}^1,v_{m2}^2}$. Thus, we have:
\begin{align}
& \angle  \mathbf{C_{m2}^1},\mathbf{v_{m2}^1}= \angle \mathbf{C_{m2}^1},\mathbf{v_{m2}^2}\\
&\angle  \mathbf{C_{m2}^2},\mathbf{v_{m2}^1}= \angle \mathbf{C_{m2}^2},\mathbf{v_{m2}^2}
\end{align}       
\end{theorem} 
From Theorems~\ref{th:mirror},\ref{th:viewdir}, we easily deduce that \begin{equation*}
%\label{eq:180degrees}
\angle \mathbf{C_{m2}^1,v_{m2}^j}+ \angle  \mathbf{C_{m2}^2,v_{m2}^j} = 180^{\circ}
\end{equation*}
\begin{corollary}
\label{cor:cheirality}
The correct one of solutions~\eqref{eq:sols} can be identified by requiring all world points that are visible from camera $2$ to be in the space in front of camera $2$. 
\end{corollary}

\section{Geometric verification of tentative image correspondences}
\label{sec:GV}
\subsection{Reduction to Longest Common Subsequence problem}
\label{sec:GVlcs}
The geometric property we pursued to enforce in tentative image correspondences is the order of imaged points with respect to the horizontal and vertical image directions.  We have:
\begin{itemize}
\item If a point, A, is imaged to the left of a point, B, in the first image, then A should be to the left of B in the second image as well. We call this property Consistency-x
\item Similarly, if a point, A, is below another point, B, in the first image, then A should be below B in the second image as well. We call this property Consistency-y 
\end{itemize}

To see how we can arrive at the LCS/LIS problem we examine each one of the two Consistency properties independently. We present here the analysis concerning Consistency-x. 

 We start with a formal definition of Consistency-x. A  set of correspondences $S=\{(p_1^i,p_2^i)\}$, where $p_j^i$ is the x-coordinate of a point ($i$) in image $j$, has Consistency-x, if for all points $p_1^i$ of image $1$ in $S$:
\begin{enumerate}
\item All points in image $1$ that are in the Consistent-x set and are to the left of $p_1^i$ match in image $2$ with points that are to the left of $p_2^i$:
\begin{equation*}
\resizebox{\linewidth}{!}{\ensuremath{
S=\{(p_1^i,p_2^i)\}:\forall i \forall j\, p_1^i\leq p_1^j \implies p_2^i\leq p_2^j,\,(p_1^i,p_2^i),(p_1^j,p_2^j)\in S}}
\end{equation*}
\item All points in image $1$ that are in the Consistent-x set and are to the right of $p_1^i$ match in image $2$ with points that are to the right of $p_2^i$:
\begin{equation*}
\resizebox{\linewidth}{!}{\ensuremath{
S=\{(p_1^i,p_2^i)\}:\forall i \forall j\, p_1^i\geq p_1^j \implies p_2^i\geq p_2^j,\,(p_1^i,p_2^i),(p_1^j,p_2^j)\in S}}
\end{equation*}
\end{enumerate}
We seek the most populous set $S$ of correspondences which is Consistent-x. We can reduce the Consistency-x problem to LIS in the following way:  We sort points $p_1^i$ in image 1 with respect to the x-axis ($x_i$). This sorting is a permutation in the sequence of correspondences. We apply this same permutation to $p_2^i$ and get a sequence from the ordinates $x_i$ of points $p_2^i$. We seek the LIS of this last sequence.

The LCS/LIS problems are efficiently solved with complexity $\mathcal{O}(n\log n)$~\cite{aldous1999longest,hunt1977fast,fredman1975computing} or even $\mathcal{O}(n\log\log n)$ \cite{van1975preserving} if special data structures are implemented. To solve LCS/LIS, we used the patience sorting algorithm~\cite{aldous1999longest}.
\subsubsection{Perplexities of the combined Consistency-x,y problem and an efficient approximate method}
In the combined Consistency-x,y problem we seek to find the largest subset of correspondences which are consistent in both x and y axes. The relation "Consistency-x and Consistency-y" is not transitive. We can observe that easily with a counter-example. Thus, the Consistency-x,y relation is not a partial order, a condition sufficient to rule out  reduction to LCS/LIS~\cite{fredman1975computing}.

Formally, in the  Consistency-x,y problem, we seek a set $S_{xy}$ of image correspondences so that:
\begin{itemize}
\item $S_{xy}$ has the Consistency-x property
\item $S_{xy}$ has the Consistency-y property
\item The number of elements ($n$) of the set $S_{xy}$ is maximized
\end{itemize}
We propose an approximate solution by the following method:
\begin{enumerate}
\item We find the largest Consistent-x subset, $S_x$, solving an LIS problem
\item We find the largest Consistent-y subset of $S_x$, solving again an LIS problem
\end{enumerate}
In our suboptimal solution Consistency-x,y holds, thus our primary aim to reject erroneous matches is achieved. Nevertheless, some true matches are rejected.
\subsection{A practical verification method}
\label{sec:GVpract}
\begin{figure}[tb]
\centering
\includegraphics[width=\linewidth]{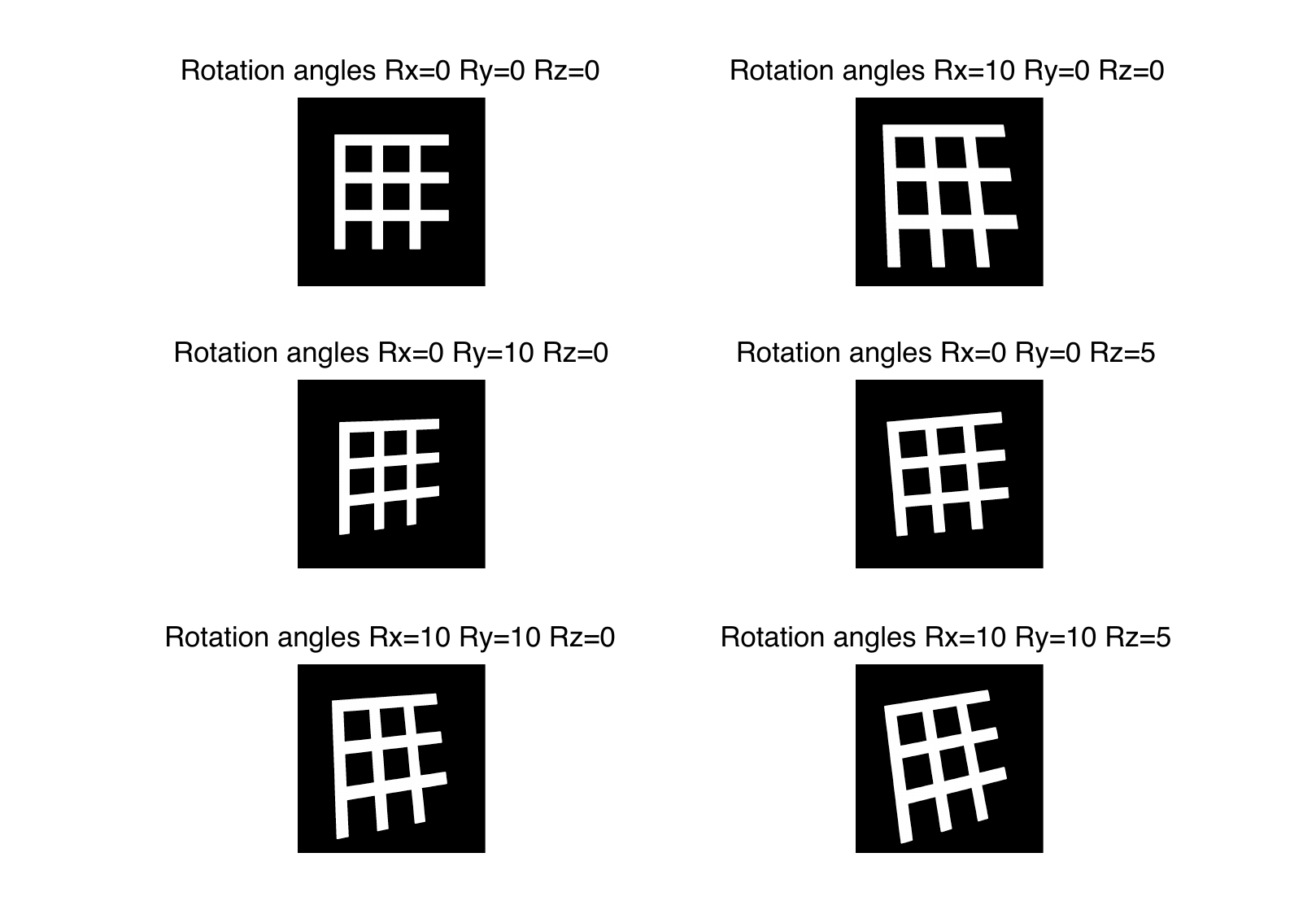}
\caption{
Effect of camera coordinate system rotation to depiction of parallel lines. The figures were created by projecting a 3D structure at a constant depth  ($z=z_{const}$) to the image plane, thus effects of different scene depth are not shown}\label{fig:3Persp}
 \end{figure}
The consistency properties we introduced, depend on assumptions on the geometric structure of the scene. In photos of architectural scenes, usually the $y-$axis in camera coordinate system aligns with the perpendicular to the floor vector, leading to the   assumption $R_z=R_x=0$. In special cases, as in photographs of houses on a street, the camera $x-$axis may also be aligned between photographs.

Such assumptions may be violated between different views. The effects of camera rotations on a scene are illustrated in Fig.~\ref{fig:3Persp}.  We observe that:
\begin{itemize}
\item Lines parallel to $x-$ or $y-$ axis in one image may appear tilted in another, if the camera coordinate systems are not aligned. The same effect is caused by scene depth variation
\item The relative order of points may change between two images. Moreover, it is more likely for two points to change order with respect to the $x$-axis, if those points are close in $x-$ axis but distant in $y$ axis, $\Delta x \ll \Delta y$.
\end{itemize}
Still, in the case of photographs of architectural scenes, we can assume small rotations around the $x,y$ axes, as the photographer's position in space is constrained. In-plane ($R_z$) rotations are uncommon and  can nevertheless be fixed automatically~\cite{gallagher2005using}.
\subsubsection{Approximations to Consistency properties}
As Consistency properties are violated by projective phenomena, enforcing them leads to the rejection of many true correspondences. Thus, we relax Consistency properties to arrive at a practical verification method. We describe the method concerned with the order of points on the $x-$axis. Similar modifications apply to the Consistency-y property.

First, we introduce a threshold value ($T$) to allow violations in the order of points that remain within a predefined distance range. So, two consecutive sequence points $s_i$, $s_{i+1}$ are considered in correct order if
\begin{equation*}
s_i -T \leq s_{i+1}
\end{equation*}
where $s_i$ is the $x-$coordinate of the i-th point in the ordered sequence.  We set $T$ as a fraction of the maximum distance in the $y-$ axis, of any two points in the image we examine, that matched to points in the paired image:
\begin{equation}
\label{eq:LCSta}
T=\alpha (y_{max}-y_{min})
\end{equation}
In the following we refer to this process  as "setting the threshold as a percentage of image size".

We propose to use a recursive method, acting on image subregions of different size. We solve a sequence of LIS problems, each one with a different $T$ value, set as a percentage of image size:
\begin{enumerate}
\item We solve an LIS problem using a threshold $T$ as a percentage of image size. The result is a Consistent-x set of correspondences
\item We split the image in two subregions, each  with equal number of correspondences $n_i$. The split is done on the $y-$axis
\item (Recursion): We repeat the process on each of the two subregions. We terminate if the region size is smaller than a predefined constant $c$ (we used $c=200$ pixels)
\end{enumerate}
The recursive method has the advantage of allowing for larger violations in the order in the $x-$axis for points that are distant in the $y-$axis, as explained in Section~\ref{sec:GVpract}. Concerning the computational complexity, we have:
\begin{small}
\begin{align*}
\text{Complexity}&=n\log n + \sum_{L}{\sum{n_i\log n_i} }\\
&\leq(L+1)n\log n,\,\text{since $n_i\leq n$, $\sum n_i\leq n$}\\
\end{align*}
\end{small}
where $L$ is the number of recursion steps. $L$ depends on the initial image size and $c$. Consequently, the recursive method adds no significant computational burden to the initial LIS problem formulation.

Finally, we remark that other approaches, as dropping recursion or fitting a simple transform to map lines between the images to estimate the $T$ value, produced worse results than the proposed recursive method.

\section{An application to the multiple-view reconstruction problem}
\label{sec:AppReconstruction}
We  integrate our methods for the geometric verification of image correspondences and the pair-based estimation of $R, f$, in existing pipelines to solve the multiple-view reconstruction problem and produce a 3D-model of a scene.

Our approach is outlined in Fig.~\ref{fig:FullPipeline}. The final reconstruction is done using the non-sequential SfM with known rotations formulation of~\cite{sfmrot2}, which we  modify  extensively, using the methods of the preceeding sections as well as the $f,R$ averaging algorithms we describe in the following.

\begin{figure*}
\begin{center}
\includegraphics[width=\textwidth]{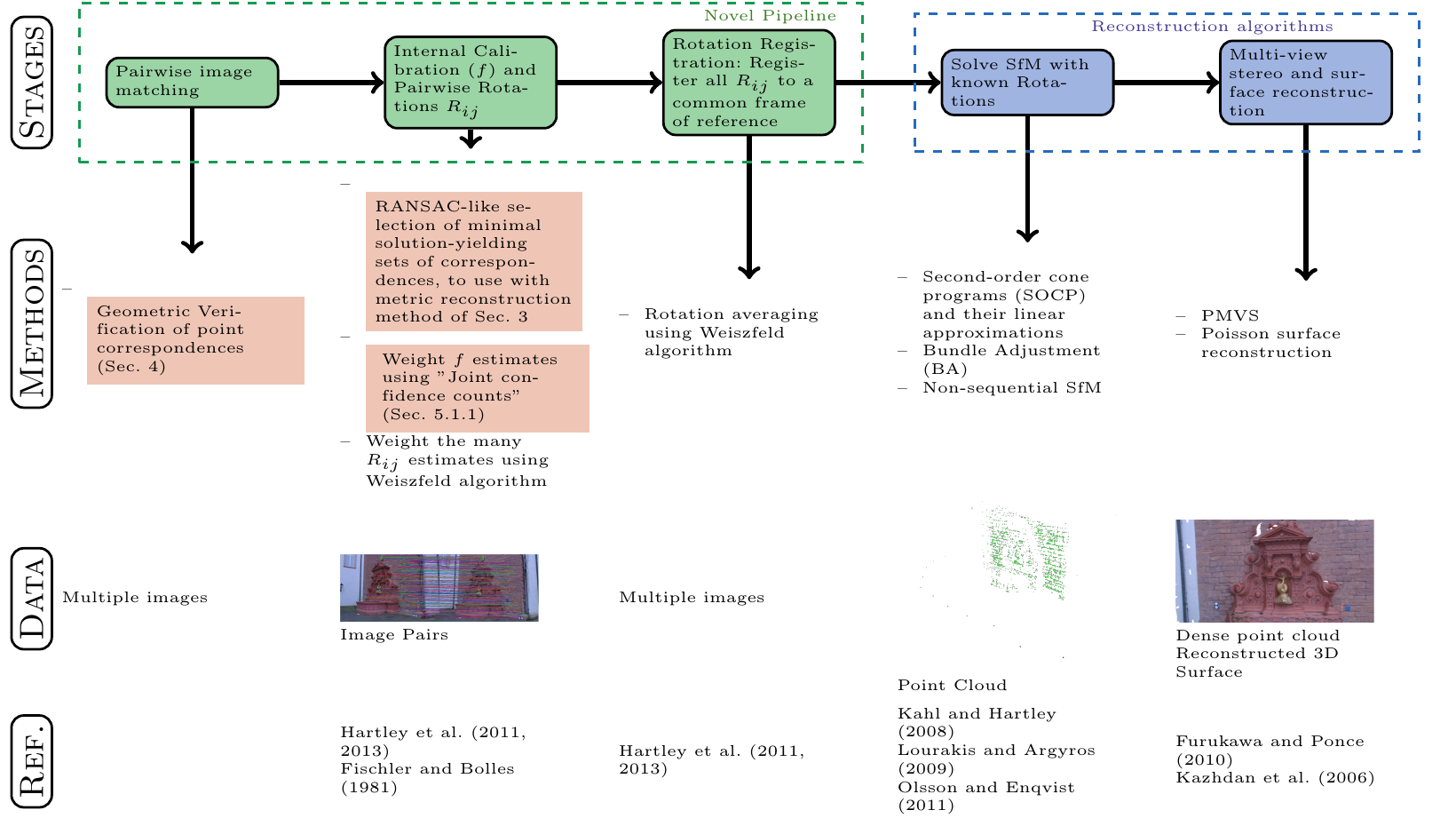}
\end{center}
\caption{Our pipeline  to reconstruct a 3D scene from an unordered set of 2D photographs. In the first row, we display a flow diagram of the algorithm \textbf{stages}. Novel parts are displayed in green. The second row outlines the core \textbf{methods} we use. We \colorbox{BrickRed!20}{highlight} methods we introduced in preceding sections. The two final rows contain a visualization of \textbf{data type} and most important \textbf{references} per stage}
\label{fig:FullPipeline}
\end{figure*}

\subsection{Averaging pair-based solutions for $f,R$}
In this paper we introduced computationally efficient methods for $R_{ij}, f_i$ estimation, which we apply in randomly sampled minimal correspondences sets, in a way that resembles RANSAC procedures~\cite{ransacalgcit}. The multiple $f_i$, $R_{ij}$ estimates, one from each minimal sample, are then averaged, to produce the final solutions.

In  $f_i$, case, we introduce a novel averaging method. In the case of pairwise rotations $R_{ij}$, we apply the Weiszfeld algorithm~\cite{rotavealg2, rotavealg1}, which converges to the median ($L_1$-average) rotation. We also use a form of the Weiszfeld algorithm (multiple rotation averaging) in the rotation registration problem to get the final camera rotation matrices $R_i$ (Section~\ref{sec:ssRotestim}).
\subsubsection{Focal length estimates}
\label{sec:AppWei}
The distribution of $f_i$ estimates collected from all the possible image pairs $i,j$ can be skewed or multimodal (Fig.~\ref{fig:TwoFdistr}), in which case the mean or median estimate will not correctly determine $f_i$ value.
\begin{figure}[tb]
\begin{subfigure}{0.48\linewidth}
\includegraphics[width=\textwidth]{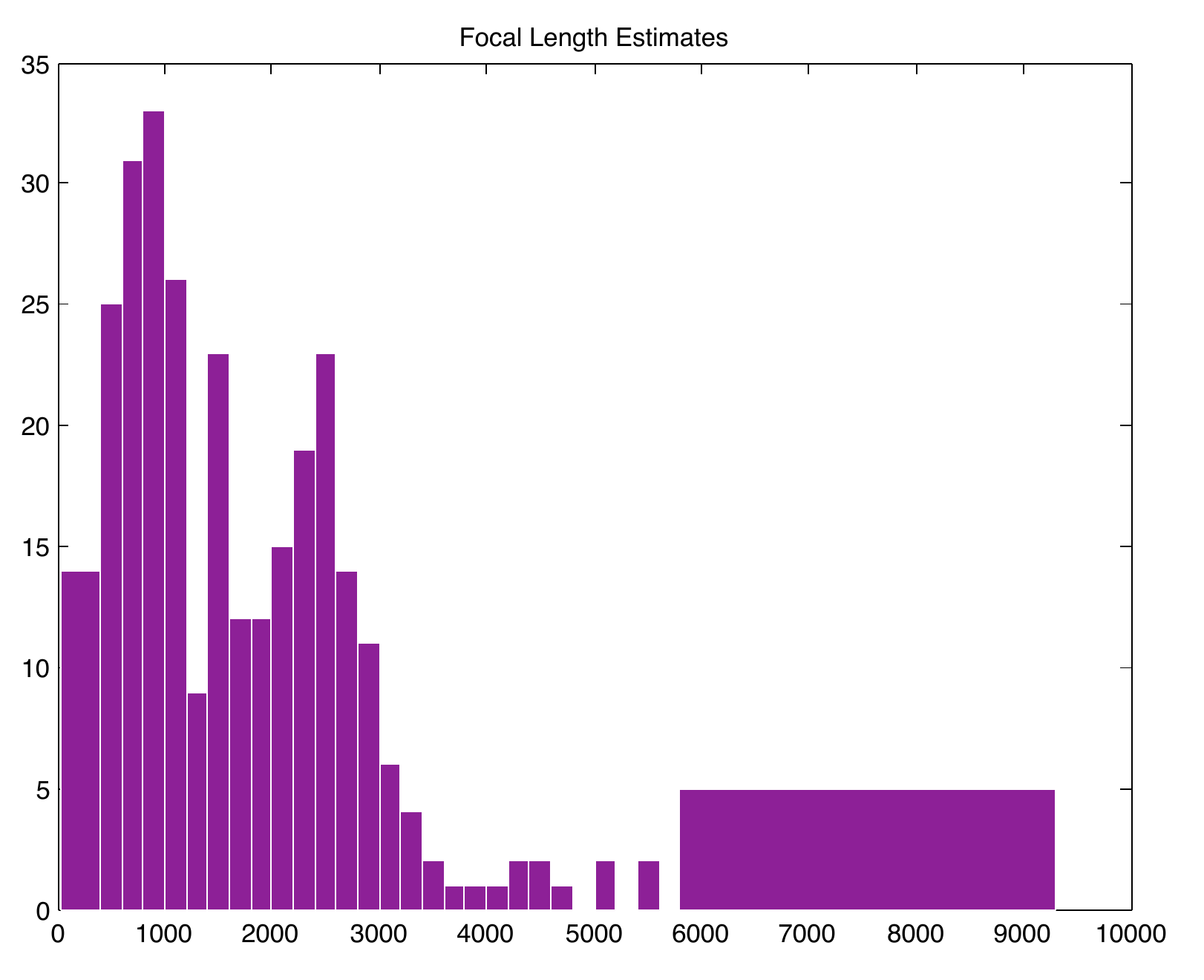}
\end{subfigure}~
\begin{subfigure}{0.48\linewidth}
\includegraphics[width=\textwidth]{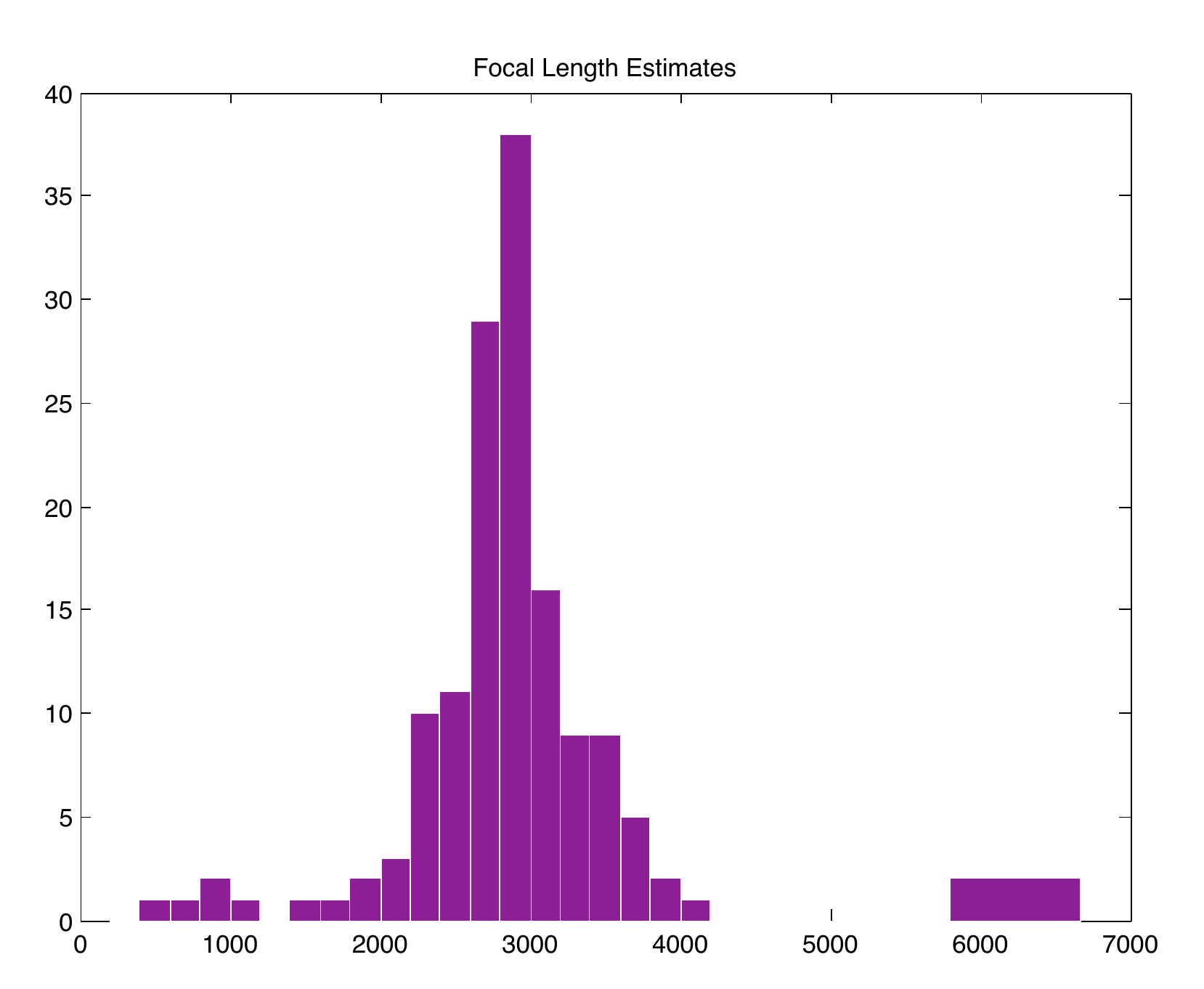}
\end{subfigure}
\caption{$f$ estimates distribution. $f$ estimates were collected from all available camera pairs. We observe that for some cameras (right), focal length can be readily determined. The opposite holds for other cameras (left). Data from castle-P30~\cite{strecha2008benchmarking} }
\label{fig:TwoFdistr}
\end{figure}

We introduce new measures to evaluate the fit of focal length estimates.  We initially introduce the Confidence count (cc) and then modify cc using the  problem structure to introduce the Joint confidence count (Jcc).
 We assume that in each image pair that contains image $i$, we receive a number of correct and a number of erroneous estimates for $f_i$, and that erroneous estimates originating from different image pairs vary significantly in value, whereas correct ones aggregate. 
\begin{figure}[tb]
\begin{center}
\includegraphics[width=\linewidth]{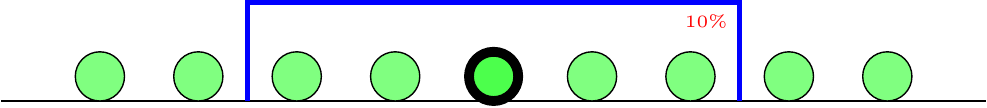}
\end{center}
\caption{A demonstration of confidence count computation. Each disk represents a $f_i$ estimate. We compute the cc for the central value, depicted here with a bold border. This cc depends on the number of estimates within a $\beta = 10\%$ range, depicted with a red square in the picture}
\label{fig:ccexplain}
\end{figure}

We visualize cc computation in Fig.~\ref{fig:ccexplain}. Simplifying aspects of the computation, we can describe it as a binning procedure, where the bin range is adapted to contain all estimates within $\beta \%$ deviation:
\begin{enumerate}
\item We collect all $f_i^n$ estimates of $f_i$, originating from all the different images we have matched with image $i$
\item For each $f_i^n$, we count the number of estimates, $f_i^k$, within a $\beta\%$  error range. This sum is the confidence count ${cc}_i^n$ for estimate $f_i^n$
\item We normalize ${cc}_i^n$ values to $0\dots 1$ range. This step is critical for Jcc computation
\end{enumerate}
To further improve the $f$ estimation, we introduce  Jcc (Fig.~\ref{fig:jccexplain}).  Since each estimate $f_i^n$ is paired with some estimate $f_j^n$  (the estimates were computed in an image pair), we expect that if $f_i^n$ is a good estimate then $f_j^n$ will be accurate too. To compute Jcc, we follow a similar to cc procedure, but this time each estimate $f_i^k$ in $\beta\%$ range contributes a different amount to Jcc sum. This amount is proportional to $cc_j^k$ of estimate $f_j^k$ that is paired with $f_i^k$. Good $f_j^k$ estimates have higher confidence counts, and contribute more to Jcc.
\begin{figure}[tb]
\begin{center}
\includegraphics[width=\linewidth]{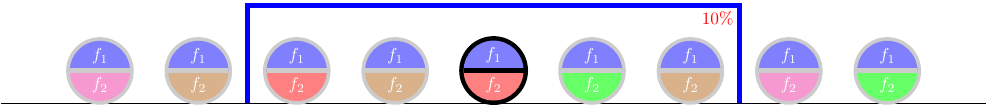}
\end{center}
\caption{A demonstration of joint confidence count computation. Each disk represents a $f_1$ estimate. We compute the Jcc for the central value, depicted here with a bold border. This time, each disk is divided in half, to demonstrate that each $f_1$ estimate is paired with one $f_2$ estimate. We use different disk colors for different cameras. Jcc depends on the sum of elements within range (inside the red square). In contrast with cc computation, each element contributes a different amount to the sum, depending on cc of $f_2$}
\label{fig:jccexplain}
\end{figure}

In greater detail, to compute the ${Jcc}_i^n$ of estimate $f_i^n$ about image $i$, we have:
\begin{enumerate}
\item Let $k=1\dots m$ be the $m$ images we matched with image $i$. For each image $k$ we have:
\begin{itemize}
\item From all estimates within $\beta\%$ range of $f_i^n$, we pick the $L$ ones that originate from pair $i,k$. 
\item Since every $f_i$ estimate originating from $i,k$ pair is matched to an $f_k$ estimate, from the $L$ estimates of the previous step we get the corresponding $L$ estimates of $f_k$
\item For each of the $L$ estimates of $f_k$, we have a confidence count ${cc}_k^n$. We get their mean. We do not use the direct sum, to diminish the influence of a large sum (large $L$) of low cc's.
\end{itemize}
\item ${Jcc}_i^n$ is the sum of the previous $m$ mean values. 
\end{enumerate}
\subsubsection{Rotation estimates}
\label{sec:ssRotestim}
In this section, we summarize rotation averaging using the Weiszfeld algorithm~\cite{rotavealg2, rotavealg1}. Weiszfeld algorithm  returns the $L_1$-mean  in a set of points in space $\mathcal{R}^n$. Many different metrics have been defined for rotation matrices~\cite{rotavealg1}. We limit our analysis here to 
\begin{equation*}
d_{geometric}(R,S) \triangleq \text{angle of rotation $RS^{-1}$}
\end{equation*}
Weiszfeld algorithm is a gradient-descent method and  is guaranteed to converge to the true $L_1$-mean in the case of single rotation averaging, as  averaging of pairwise rotation estimates $R_{ij}$.

The $L_1$-mean of  $R_i$ estimates of a single rotation is the rotation $R_y$ that minimizes:
\begin{equation*}
\sum_{i=1}^nd_{geometric}(R_i,R_y)
\end{equation*}
 In this case of rotation registration, the convergence of Weiszfeld algorithm is not guaranteed.

In detail, we applied Weiszfeld algorithm to weight the estimates $R_{ij}$ of the pairwise rotations we acquired through random sampling of minimal point sets ($8$ points) yielding a $R_{ij}$ solution.

In the rotation registration problem we applied the Weiszfeld algorithm in the following manner:
\begin{enumerate}
\item  We construct the rotations graph, with one node for every image and an edge $e_{ij}$ between nodes $i,j$ if we know the relative rotation $R_{ij}$ between the respective images. We take a spanning tree in this graph, and using $R_j = R_{ij}R_i$ we get the initial $R_j$ estimates 
\item For every node $i$ in the graph, we use all available estimates $R_{ij}$ to get inconsistent estimates $R_i^k$, $k=1,2,\dots$ through  $R_i = R_{ji}R_j$. We average estimates $R^k_i$ with one iteration of Weiszfeld algorithm
\item We repeat the previous step $n$ times ($n=20$)
\end{enumerate}
In all our experiments we set $\beta=10\%$.

%Draft: comparisons between Ch5 and literature
%(1) [195]: a4notes =>dn einai kruppa eq's.....
%(2) [176]: a4notes
%(3) Kruppa+5P alg: a4notes+p233draft+Tab7.10
%ADD before 195,176 details: Paragraph " Other approaches using the DIAC"  from outline

%>Sections from draft thesis
%*intro of Ch5, p116draft
%*Sec. "θετικα στοιχεια διατυπωσης...'' p122draft

%>What to mention in RW section?: brief mention, point to discussion section (this one)
%Draft: Results metric self calib
%f. length: on par with Kruppa (comment on pg159draft)
%external params: 5p method (tab 7.10)

\section{Results \& Discussion}
\label{sec:Res}
\subsection{Metric Reconstruction in Pairs of Images}
\label{sec:ResRec}
We implemented  Kruppa equations, a  well-studied and popular method for camera self calibration,  and used it as reference method  for the estimation of internal camera parameters. To compare the methods, we used synthetic camera projection matrices and image correspondences. We added Gaussian noise to the image points positions, and not to world points or other entities, to simulate actual noisy correspondences.

We observed that our method (Sec.~\ref{sec:Rec})   and the Kruppa method  produce identical $f_i$ estimates. In rare cases with extremely noise-corrupted correspondences, our method failed and the Kruppa method produced largely erroneous focal length estimates.

 We conclude  that the two methods are equivalent concerning the self-calibration problem. Still  our method is advantageous in additionally providing a  metric reconstruction.
 
Next, we evaluate camera pair reconstruction.  We compared our method to the 5-Point(5P) algorithm~\cite{5palgcit}. We used both approaches as initialization to BA~\cite{BAalgcit} and evaluated the quality of the final reconstructions (Tab.~\ref{tab:Compare5PL1}). We used  a multiple-view dataset~\cite{datasetcit} and determined the relative positions of all camera pairs with point correspondences. The same focal length estimates were used in both compared approaches. $f$ estimates  were obtained by the method we introduced in Sec.~\ref{sec:Rec}. The two methods we compared were:
\begin{itemize}
\item Initialize BA with our method: We randomly sampled minimal subsets of correspondences  and averaged the acquired solutions with rotation averaging~\cite{rotavealg2}. We allowed for 20 BA iterations
\item Initialize BA with 5P algorithm: We used a RANSAC procedure to sample minimal 5P subsets and to pick the solution. We allowed  for 20 BA iterations
\end{itemize}
To quantify the reconstruction error, we used the angle ($\Delta R$) between the relative rotation estimate and the true relative rotation, $R_{ij}$, between two paired views $i,j$.

The initialization of BA is important, to improve convergence and to reduce the computational cost. We observe that both the 5P algorithm and our method can be used as BA initialization with similar performance (Tab.~\ref{tab:Compare5PL1}). This  result implies that to further reduce the reconstruction error, we should improve other problem parameters as image correspondences and focal length estimates.

\begingroup
\hyphenpenalty 10000
\exhyphenpenalty 10000
\begin{table*}
\footnotesize
\begin{center}
\begin{tabularx}{0.99\textwidth}{R  X X X}
\midrule[1.5pt]
  & {\bf Median} $\Delta R (^{\circ})$ &  $\Delta R<10^{\circ}$ (pairs)&$\Delta R<5^{\circ}$ (pairs) \\ \midrule[1.5pt] 
 \textbf{Proposed Method}&$3.49$&$123$&$106$\\ \midrule[0.25pt]
\textbf{5P}&$3.77$&$128$&$104$\\
\lasthline
\end{tabularx}
\caption{Performance of the 5P algorithm and our method in recovering the relative rotations $R_{ij}$. We initialize BA with each of the aforementioned methods and do  $20$ BA iterations. In BA, the internal parameter matrices $K_i$ are held constant. Dataset  \textbf{castle-P30}~\cite{datasetcit}.}
\label{tab:Compare5PL1}
\end{center}
\end{table*}
\endgroup 
\subsection{Geometric verification of tentative correspondences}
\label{sec:ResGV}
\begin{figure*}
\begin{subfigure}{0.48\textwidth}
\includegraphics[width=\textwidth]{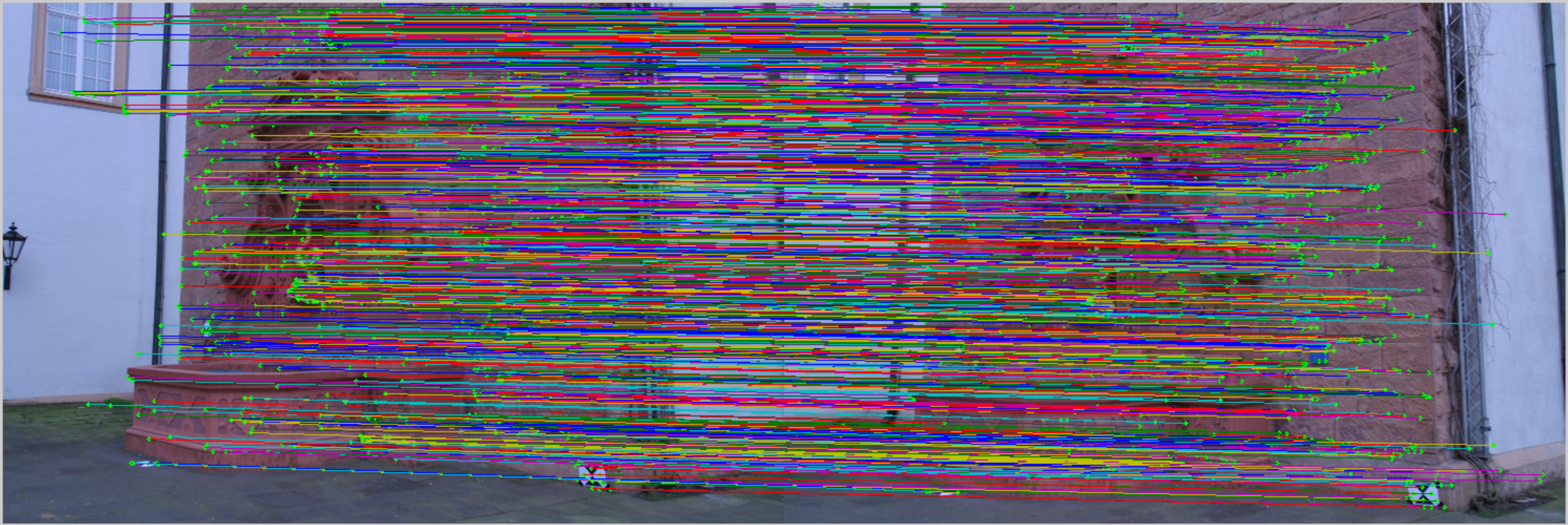}
\end{subfigure}~
\begin{subfigure}{0.48\textwidth}
\includegraphics[width=\textwidth]{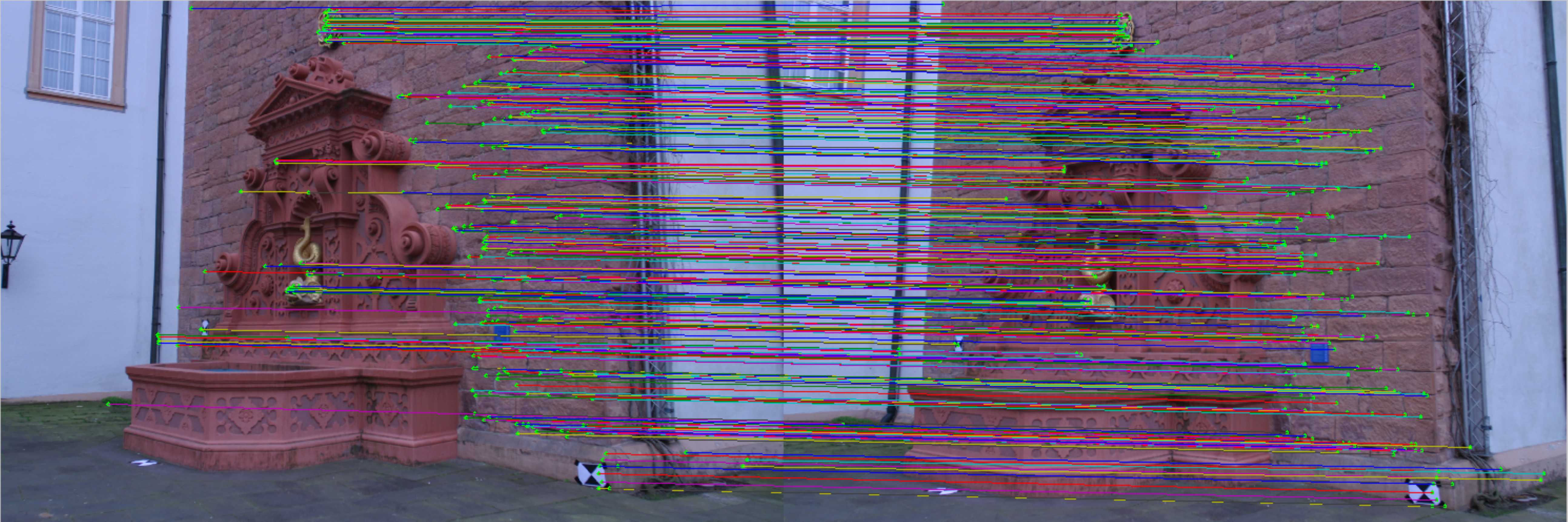}
\end{subfigure}
\caption{Demonstration  of tentative correspondences verification. Left: Initial correspondences. Right: Verified correspondences using  the geometric verification method we introduced}
\end{figure*}

In geometric verification,  correspondences are classified as correct or erroneous. We evaluate this classifier using  precision and recall. Two different datasets were used:
\begin{itemize}
\item Dataset~\cite{datasetcit}: The set contains outdoor scenes of landmark buildings. The ground truth camera matrices, $P$,  are provided, from which we can separate correct and erroneous  correspondences. In detail, from given $P_i, P_j$ we recover the fundamental matrix $F_{ij}$ and then evaluate Sampson's approximation to geometric error for each tentative correspondence~\cite{sampsoncit}
\item Dataset~\cite{wongiccv2011}: This set contains both architectural scenes and scenes with objects. We give  performance results on each of those two subsets independently. In this set, point correspondences between images are provided and labeled as correct or erroneous
\end{itemize}
The results are presented in Tables~\ref{tab:LCSresult1}, \ref{tab:LCSresult2}. Precision of the classifier is more important than its recall, as it is more important to have an oulier-free set of correspondences than to recover all true correspondences. Furthermore, the recursive verification method  discards erroneous matches with very high precision. This result supports our argument that points more distant in the one, e.g $x$, axis are more likely to violate order with respect to the other, e.g $y$, axis. Finally,  performance varies with scene type. The development of our method was based on  scene properties found in architectural scenes.  In scenes composed of objects, differences in scene geometry and the increased freedom in viewer's position cause more violations in the Consistency properties. In Dataset~\cite{datasetcit}, the performance in scenes of one main building, and consequently of a single main horizontal and vertical direction, as Fountain-P11, entry-P10, Herz-Jesu-P8, reaches almost flawless precision ($1$). In more complex scenes, which include more  buildings, as in castle-P30, the achieved precision degrades to values in the range $0.7-0.8$.

\begingroup
\hyphenpenalty 10000
\exhyphenpenalty 10000
\begin{table*}
\scriptsize
\begin{center}
\begin{tabularx}{0.99\textwidth}{>{\hsize=0.25\textwidth\raggedright\arraybackslash}R X X X X X X X}
\midrule[1.5pt]
{\bf Performance Measure} &{$\mathbf{ \alpha=0.02} $} & {$\mathbf{ \alpha=0.04} $} & {$\mathbf{ \alpha=0.06} $ } & $\mathbf{ \alpha=0.08} $ &  $\mathbf{ \alpha=0.10} $ & $\mathbf{ \alpha=0.15} $ & $\mathbf{ \alpha=0.20} $\\ \midrule[1.5pt] 
 {\bf Precision }&\textbf{0.99}{\color{RubineRed} \tiny{\emph{0.98}}}&0.98{\color{RubineRed} \tiny{\emph{0.97}}}&0.97{\color{RubineRed} \tiny{\emph{0.96}}}&0.96{\color{RubineRed} \tiny{\emph{0.96}}}&0.95{\color{RubineRed} \tiny{\emph{0.95}}}&0.91{\color{RubineRed} \tiny{\emph{0.91}}}&0.87{\color{RubineRed} \tiny{\emph{0.87}}} \\ \midrule[0.25pt]
{\bf Recall }&0.80{\color{RubineRed} \tiny{\emph{0.64}}}&0.89{\color{RubineRed} \tiny{\emph{0.72}}}&0.93{\color{RubineRed} \tiny{\emph{0.76}}}&0.94{\color{RubineRed} \tiny{\emph{0.78}}}&\textbf{0.96}{\color{RubineRed} \tiny{\emph{0.80}}}&0.95{\color{RubineRed} \tiny{\emph{0.80}}}&\textbf{0.96}{\color{RubineRed} \tiny{\emph{0.81}}} \\  
\lasthline
\end{tabularx}
\caption{ Precision and Recall for geometric verification of tentative correspondences in Dataset~\cite{wongiccv2011}. In {\color{RubineRed} \tiny{\emph{small italics} }} we give results on the complete set and, in the usual font, on the subset of architectural scenes. We applied the recursive verification method with Threshold $T$ as a percentage of image size. We give the percentage values $\alpha$ for $T$, following Eq. (\ref{eq:LCSta}) }
\label{tab:LCSresult1}
\end{center}
\end{table*}
\endgroup

\begingroup
\hyphenpenalty 10000
\exhyphenpenalty 10000
\begin{table*}
\scriptsize
\begin{center}
\begin{tabularx}{0.99\textwidth}{>{\hsize=0.25\textwidth\raggedright\arraybackslash}R X X X X X X X}
\midrule[1.5pt]
{\bf Performance Measure} &{$\mathbf{ \alpha=0.02} $} & {$\mathbf{ \alpha=0.04} $} & {$\mathbf{ \alpha=0.06} $ } & $\mathbf{ \alpha=0.08} $ &  $\mathbf{ \alpha=0.10} $ & $\mathbf{ \alpha=0.15} $ & $\mathbf{ \alpha=0.20} $\\ \midrule[1.5pt] 
{\bf Precision }&0.60 &0.70 &0.76 &0.79 &0.82 &0.86 &\textbf{0.88} \\ \midrule[0.25pt]
{\bf Recall}&0.85 &0.85 &0.85 &0.85 &0.85 &0.85 &0.84 \\ 
\lasthline
\end{tabularx}
\caption{ Precision and Recall for geometric verification of tentative correspondences in Dataset~\cite{datasetcit}. We applied the recursive verification method with Threshold $T$ as a percentage of image size. We give the percentage values $\alpha$ for $T$, following Eq. (\ref{eq:LCSta}) }
\label{tab:LCSresult2}
\end{center}
\end{table*}
\endgroup

\subsection{Improving focal length estimation in multi-view reconstructions}
 We show in Tab.~\ref{tab:ConfidenceEval} the improved $f$ estimates we get with cc. Further improvement is achieved by Jcc measure. To quantify the error in $f$ estimation we use $\Delta f$~\cite{chandraker2007autocalibration, gherardi2010practical, kukelova2008polynomial}:
\begin{equation*}
\Delta f \triangleq \left| \frac{\hat{f}}{f}-1 \right|,\quad\text{where $\hat{f}$ is an estimate of $f$}
\end{equation*}
\begingroup
\hyphenpenalty 10000
\exhyphenpenalty 10000
\begin{table*}
\scriptsize
\begin{center}
\begin{tabularx}{0.99\textwidth}{>{\hsize=0.25\textwidth\raggedright\arraybackslash}R X X X}
\midrule[1.5pt]
{\bf Method } & {\bf Median } &{\bf Confidence count } &{\bf Joint confidence count } 
\\ \midrule[1.5pt]
{\bf Mean $\Delta f$ error } & $0.28 $ & $0.17$ & $0.07$
 \\  
\lasthline
\end{tabularx}
\caption{Focal length averaging on \textbf{castle-P30}~\cite{strecha2008benchmarking}. We did $300$ RANSAC iterations and used the 8-point algorithm for fundamental matrix computation}
\label{tab:ConfidenceEval}
\end{center}
\end{table*}
\endgroup

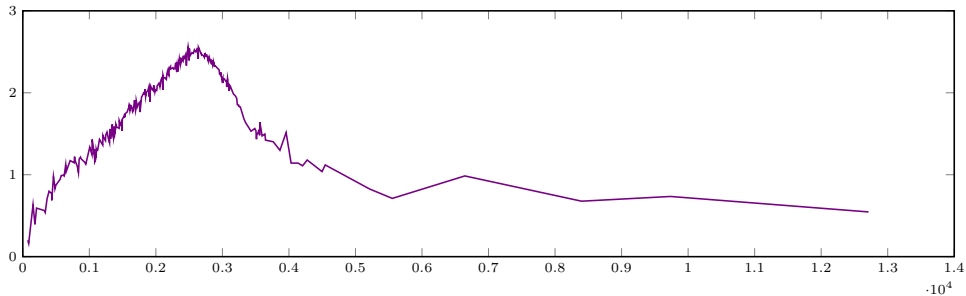
\begin{figure}[tbp]
\scriptsize
 \centering
\newlength{\figurewidth}
\setlength{\figurewidth}{\linewidth}
\newlength{\figureheight}
\setlength{\figureheight}{0.2\textheight}
\scalebox{.75}{

% This file was created by matlab2tikz v0.4.3.
% Copyright (c) 2008--2013, Nico Schlömer <nico.schloemer@gmail.com>
% All rights reserved.
% 
% The latest updates can be retrieved from
%   http://www.mathworks.com/matlabcentral/fileexchange/22022-matlab2tikz
% where you can also make suggestions and rate matlab2tikz.
% 
%
% defining custom colors
\definecolor{mycolor1}{rgb}{0.4706,0,0.5216}%
\begin{tikzpicture}

\begin{axis}[%
width=\figurewidth,
height=\figureheight,
scale only axis,
xmin=0,
xmax=14000,
ymin=0,
ymax=3
]
\addplot [
color=mycolor1,
solid,
forget plot
]
table[row sep=crcr]{
75.2981624098223 0.20037474671435\\
88.0535786031209 0.161992406804063\\
153.534055792101 0.633967462647787\\
182.303852524711 0.394565792513697\\
206.222747250526 0.592610039878298\\
324.432751555263 0.562492667897229\\
337.041433296921 0.539890917435407\\
364.518472542779 0.706320110905324\\
394.599439367372 0.798475230113076\\
434.960080743564 0.768220570849496\\
435.165775335726 0.687253272553432\\
460.648361677842 0.978066056709746\\
474.207778881904 0.90109005302844\\
483.171905890169 0.828533950378567\\
496.893441952408 0.872352924827019\\
560.076625748314 0.943433556458362\\
577.793536441919 0.990472133386349\\
611.385766371385 0.993066831776019\\
621.195163244918 0.985198523133169\\
644.650844070495 1.11978195640288\\
651.883070560736 1.03682666469032\\
711.629501866509 1.17251658034411\\
772.948966654022 1.14534900309663\\
777.885019461768 1.2201477244735\\
783.592568133507 1.1778612842867\\
815.319358428387 1.1102568092558\\
825.873276975678 1.08063223914774\\
829.467777381753 1.06100268168299\\
836.37484661111 1.02702030361684\\
847.459118880097 1.19110815981625\\
867.170491148001 1.21447022645349\\
874.350716767299 1.19625452016927\\
920.661085829362 1.15871256947621\\
927.47047756383 1.1557866521715\\
946.625022963726 1.1284363098739\\
1005.67715302267 1.33565627125126\\
1039.27901975616 1.24475414380846\\
1042.83112071132 1.43034271137514\\
1071.81630064249 1.26503879732855\\
1075.00163157882 1.30692977367265\\
1081.20111840743 1.15689782987907\\
1094.37891750656 1.18436160070759\\
1098.6587233507 1.27312502639318\\
1101.04814420787 1.21145189572588\\
1103.46295362633 1.30909454333324\\
1109.76535935591 1.28798628110105\\
1115.65745317854 1.30801698920964\\
1125.85177697482 1.30393700513704\\
1152.1946608204 1.43111127756173\\
1201.61437219944 1.36241581678101\\
1204.35143123477 1.47942782332102\\
1214.26625381682 1.45782896333216\\
1242.51176514625 1.41724969527191\\
1246.17288894333 1.48214346491813\\
1265.15003805107 1.51970898651553\\
1285.01650582203 1.44749857842725\\
1303.47293271382 1.39101567896505\\
1307.79859365019 1.55911158616399\\
1315.06355986828 1.52033007487652\\
1318.93130529187 1.43808611390619\\
1345.25117512275 1.61574816335162\\
1346.1355281453 1.53983443653062\\
1362.16573196453 1.45282131802152\\
1371.08667699716 1.58799446978107\\
1382.99803166163 1.50146800838567\\
1397.0160708887 1.58014278922318\\
1398.3346180463 1.61471784664633\\
1412.23749917746 1.5779790961608\\
1432.08551336206 1.57383268063255\\
1447.16219483578 1.56652966377497\\
1449.16734328852 1.66162350715289\\
1469.42624632459 1.61135711210583\\
1486.88804588816 1.64801970191164\\
1490.58735604622 1.53721356025201\\
1496.02045590454 1.6778960788647\\
1528.02975095445 1.72107588117523\\
1530.44843676582 1.73763089994951\\
1531.17458003226 1.69973002531395\\
1531.90404024975 1.73994685997565\\
1568.70337942938 1.7648858282171\\
1587.63615697098 1.8115941054043\\
1594.0122983032 1.81720439838812\\
1600.17208768656 1.85858372459065\\
1602.56229334243 1.85501864509439\\
1610.58333574484 1.78124323204153\\
1627.11295268781 1.84190374644177\\
1630.31180060664 1.83958096924869\\
1640.43257614699 1.81573818110734\\
1646.01986652281 1.78063148724704\\
1646.98821312514 1.77911056508792\\
1651.40133323124 1.7679283967144\\
1662.2483664999 1.7953750099919\\
1663.44027661872 1.80802110893068\\
1676.91014233924 1.89309400543123\\
1678.01122727597 1.90916609227234\\
1694.84926629947 1.78553619322844\\
1702.39399549554 1.83202146582669\\
1705.20374167692 1.90374766080065\\
1722.96906023763 1.83018902464224\\
1744.90282389306 1.87308136930649\\
1760.30327675836 1.8885628132985\\
1762.64342671002 1.76426936418195\\
1784.8476098769 1.94999404286227\\
1792.33449859458 1.95758361419672\\
1817.28329240341 1.99946646901701\\
1830.8272388069 1.99032873584023\\
1831.67214302912 1.97283602446997\\
1833.7697662216 2.04304579580144\\
1835.3411869486 2.00242094154944\\
1841.43950828733 2.00208984858854\\
1845.82339986936 1.95714315941347\\
1850.95743650123 1.99032641291858\\
1860.816154548 2.00720710733167\\
1882.18283209765 2.03426817204339\\
1893.46719612055 2.06792277547046\\
1900.06344808889 1.95244258790258\\
1903.97799332369 2.0378250145281\\
1908.22492974953 2.07060878407186\\
1910.35244124586 1.89149594038013\\
1911.47131650052 2.03544174563673\\
1916.80680380615 2.0881471716706\\
1941.60863577621 2.04177746932341\\
1947.91261220306 2.04787441533535\\
1962.83570004292 2.02743966488386\\
1967.81511975765 2.09074230922231\\
1971.22316066227 2.03887112980864\\
1974.50379342045 2.08268609471545\\
1985.52571307532 2.00907549544388\\
1991.38764200014 2.0310334828219\\
1994.87677927076 2.03542276119444\\
2011.66025116312 2.02727789075967\\
2016.77986096602 2.04997182148133\\
2017.63139095192 2.0805210852499\\
2046.55373366759 2.11726678031118\\
2060.85763841458 2.07156408906601\\
2072.92778744102 2.13117845186721\\
2076.35268077412 2.15695703482206\\
2096.25291153364 2.18787608477473\\
2103.83179374803 2.03378963310819\\
2107.39958790347 2.21358061106423\\
2116.4606350841 2.18083199759245\\
2118.49681358382 2.18279454975969\\
2134.68964142288 2.18370527965535\\
2157.4808373061 2.16302798957018\\
2163.91718611155 2.21524328504564\\
2169.71966100741 2.24805074977918\\
2180.09830830379 2.27542590063443\\
2182.00247542905 2.22995818875741\\
2196.45115108716 2.20894251119811\\
2198.16618789009 2.29067773824435\\
2219.49946503791 2.31926787975432\\
2219.52950401133 2.28844912368046\\
2253.70791827084 2.3051248323421\\
2256.90573706007 2.3069782774461\\
2267.33412112934 2.28916925957811\\
2285.39357056601 2.2946354089671\\
2292.79504904037 2.34346801308523\\
2297.40402593055 2.33949614328645\\
2303.73475871261 2.30536753043975\\
2304.98998881991 2.31358105461481\\
2309.2432842839 2.36986767483035\\
2309.27680753586 2.37179639913476\\
2310.64706065057 2.32037432941035\\
2315.7787956263 2.34492185366356\\
2323.20833598925 2.25604620076285\\
2327.7945274303 2.30763705429004\\
2328.10697095216 2.32644583171951\\
2335.75000056615 2.26278080247355\\
2335.94806242617 2.32983619605955\\
2346.95925389687 2.39310104181063\\
2348.34594293357 2.36868646716241\\
2353.22838468245 2.41587199453952\\
2357.37697099379 2.39545398526449\\
2364.74288751545 2.34842293025979\\
2370.33279261024 2.38772485513872\\
2372.62699067426 2.42859526390706\\
2378.29323457045 2.37303230561524\\
2384.95567007642 2.39749866310314\\
2393.86127986197 2.40533429896203\\
2407.81997413298 2.38719095759376\\
2414.19248689964 2.45199702591023\\
2414.58901257638 2.40646133595769\\
2420.3160828053 2.41198946294454\\
2448.54050601462 2.47024880891421\\
2450.6330049355 2.41005594808474\\
2458.73638910408 2.33182533904972\\
2469.45253313911 2.44819173795952\\
2485.13346189439 2.56277732616784\\
2493.24698781553 2.52109168708748\\
2499.12571914695 2.39472060196494\\
2503.66847416682 2.53950112153514\\
2514.08612692798 2.45115917763209\\
2542.0032552484 2.48949201400604\\
2547.97251992643 2.48617111606611\\
2564.25194678614 2.48082407361148\\
2579.42607334041 2.52454521285889\\
2601.35844490138 2.49897283844381\\
2617.40320367151 2.53351330072746\\
2635.83194937339 2.51243800320335\\
2636.73716775476 2.41340618855368\\
2639.34276565102 2.55570778278837\\
2682.72727882181 2.47149684506098\\
2683.99679645365 2.47620501549463\\
2685.81109338343 2.47492804702699\\
2689.87861294342 2.46744291407368\\
2721.44317317803 2.4480617623026\\
2726.62599788454 2.43506605922531\\
2728.10037328224 2.43193126298707\\
2731.24995519571 2.44313479133563\\
2735.98011189151 2.48049400457924\\
2736.16518615352 2.48607728071955\\
2741.60452898577 2.45024160354161\\
2742.09834163988 2.45024160354161\\
2746.02290928297 2.46985725155269\\
2762.11070671565 2.44908262291987\\
2767.64295179708 2.45368084444105\\
2787.24849846785 2.38814632320064\\
2792.62078254671 2.42525718457802\\
2801.7689899594 2.4081274597961\\
2802.29692429845 2.39357164147253\\
2806.67022854414 2.41899686390435\\
2816.69113647762 2.37737945120172\\
2830.53767890043 2.36095682345459\\
2833.57667188811 2.40434284664438\\
2836.85845584267 2.38229480053447\\
2843.66563621757 2.38500343368846\\
2846.71063745765 2.36254678556362\\
2854.79492564991 2.35338869709895\\
2857.05128379119 2.33144512343123\\
2868.54598421955 2.35638305057603\\
2868.76177635884 2.37240109476901\\
2875.98374464197 2.35008689127108\\
2881.04344649509 2.32507361577854\\
2892.96183029325 2.32864527116877\\
2950.7184325801 2.27910997179242\\
2956.99288941817 2.23148280557591\\
2974.58484424348 2.21346890114581\\
2975.5201459501 2.24952894683274\\
2993.79827603112 2.15197609409409\\
3006.64561639163 2.18236635881372\\
3007.30237138953 2.11842326938283\\
3015.11738938717 2.17067546208537\\
3030.02575177249 2.16902934542211\\
3050.20715916404 2.14195770361282\\
3061.80919187816 2.14065225480141\\
3064.64904282741 2.0697154765901\\
3076.89043647369 2.17327605692727\\
3085.15359565807 2.11193188835372\\
3096.07264666566 2.10153981412823\\
3097.56055437333 2.02338999086462\\
3103.50858034378 2.07440230161137\\
3110.87966187762 2.05326373239666\\
3120.75177940209 2.07548126453666\\
3166.38472642924 1.98437333720982\\
3194.39683156547 1.96262673114671\\
3215.2028400042 1.93172791632934\\
3224.45841189063 1.85371000915049\\
3245.43733675787 1.84918256337869\\
3245.71618473311 1.83440786276301\\
3269.44825108388 1.82625046588931\\
3321.17253890089 1.68700726483493\\
3345.21505764642 1.64031152901268\\
3430.4334778352 1.53077970861572\\
3487.03179064153 1.564008524172\\
3504.2334344777 1.53591421918264\\
3512.36005768715 1.4333438072965\\
3514.46643419005 1.47741236391212\\
3549.35110113855 1.54031137980648\\
3559.56404119478 1.48089347159772\\
3563.46867483889 1.6396366357112\\
3593.26366898892 1.4903000497531\\
3596.89704798563 1.47332353566676\\
3638.65751281438 1.4963080724479\\
3648.15579711148 1.42164006558208\\
3764.2156358719 1.40238399566309\\
3864.44659609184 1.29805895095292\\
3957.65697135369 1.51429152758698\\
4032.51363312257 1.14188567353112\\
4136.36040955944 1.14198415303845\\
4205.49026688101 1.10885374945161\\
4274.35563305649 1.17941549054908\\
4497.54744722981 1.03809691945862\\
4545.71694978764 1.11927930263925\\
5212.11763081286 0.826041823594442\\
5553.79922064832 0.711040200232446\\
6644.32910985729 0.985214909501321\\
8398.66813760495 0.67698727273796\\
9735.55989438079 0.73464724360911\\
12711.8988260534 0.545882144746254\\
};
\end{axis}
\end{tikzpicture}%
}
\caption{Joint confidence count  distribution, which displays a clear peak near  correct $f$ value. Data from \textbf{castle-P30}~\cite{strecha2008benchmarking}}
\label{fig:JccPeaked}
\end{figure}

 In  Fig.~\ref{fig:JccPeaked} we present an non-ambiguous Jcc distribution from which $f$ can be correctly determined, in contrast to $f$ estimates distributions displayed in Fig.~\ref{fig:TwoFdistr}. 

 \subsection{Multi-view reconstruction in unordered image sets}
In Tab.~\ref{Tab:ReconstructionErrors}, we provide quantitative performance measures for multi-view reconstructions that were acquired applying the proposed pipeline (Fig.~\ref{fig:FullPipeline}),  on unordered image datasets and with no other input apart from the scene photographs. To quantify reconstruction error in camera translation, we used the angle ($\Delta \mathbf{t}$) between the translation estimate and the available true translation $\mathbf{t}$. In Fig.~\ref{fig:TwoReconstructions} we qualitatively display the results of the proposed reconstruction pipeline. The results in Tab.~\ref{Tab:ReconstructionErrors} and Fig.~\ref{fig:TwoReconstructions} demonstrate that the introduced methods can be used in unordered image sets to produce quality reconstructions of the photographed scenes.
\begin{figure*}
\begin{subfigure}{0.48\textwidth}
\includegraphics[width=\textwidth]{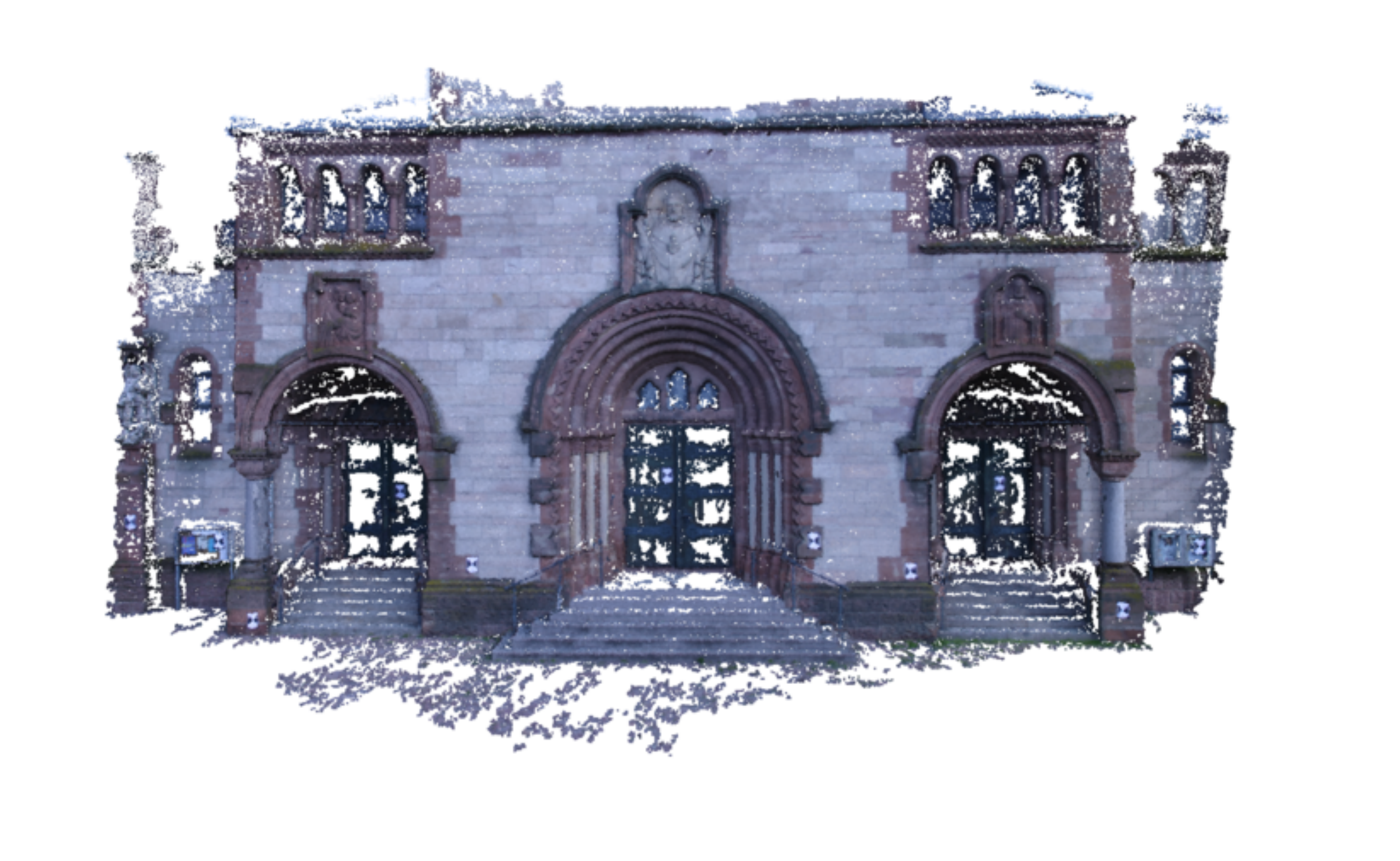}
\end{subfigure}~
\begin{subfigure}{0.48\textwidth}
\includegraphics[width=\textwidth]{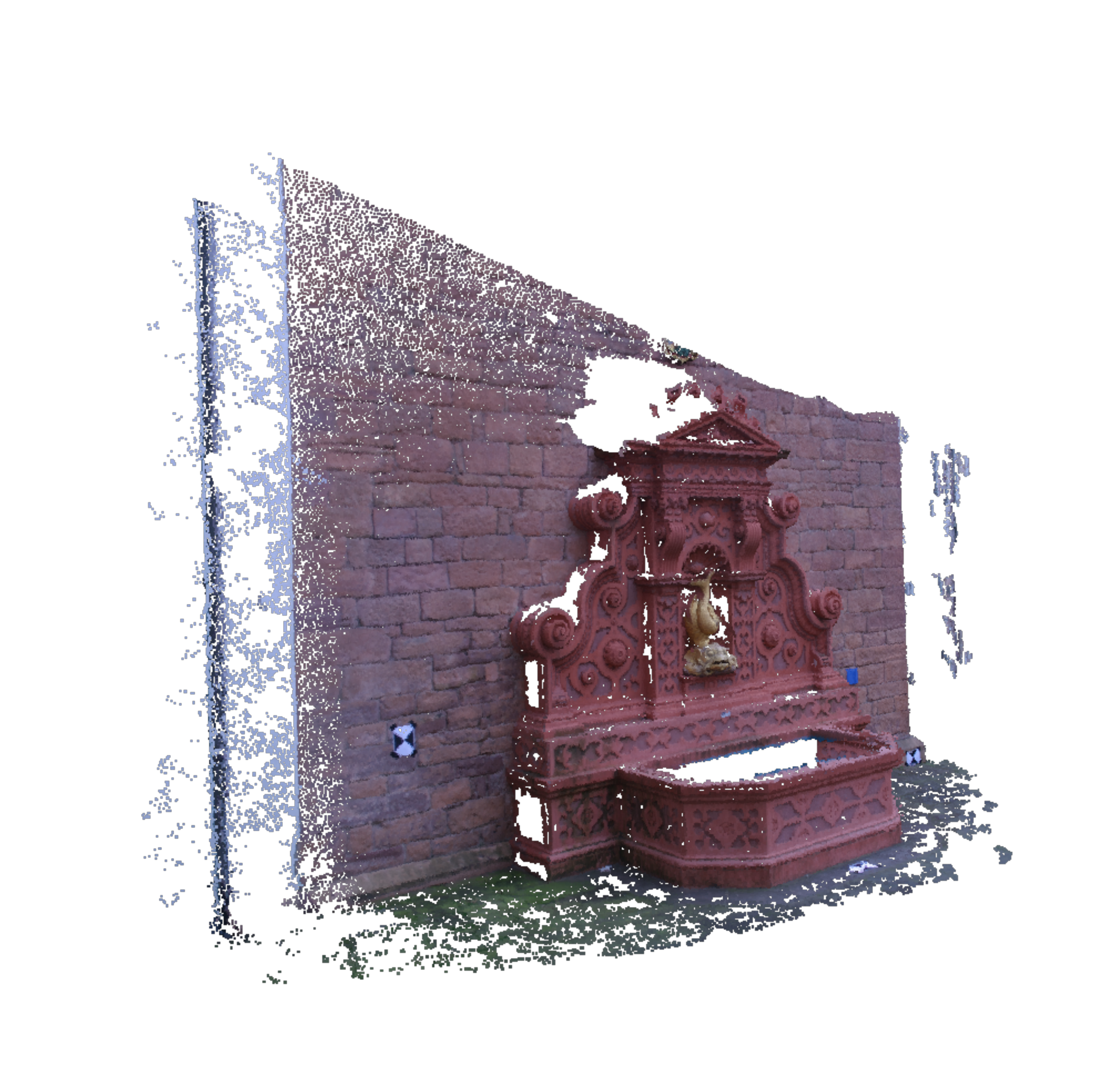}
\end{subfigure} \\
\begin{subfigure}{0.48\textwidth}
\includegraphics[width=\textwidth]{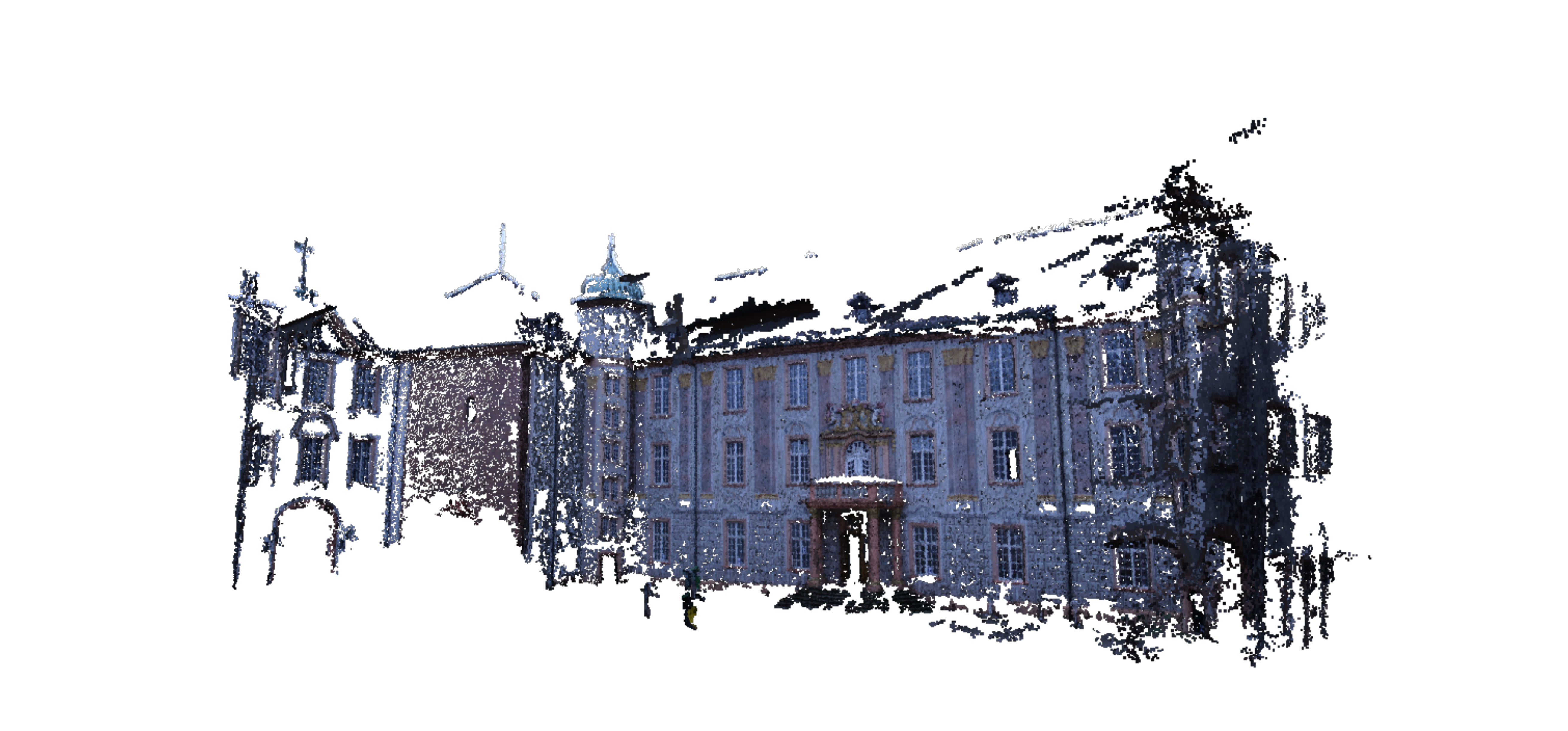}
\end{subfigure}~
\begin{subfigure}{0.30\textwidth}
\includegraphics[width=\textwidth]{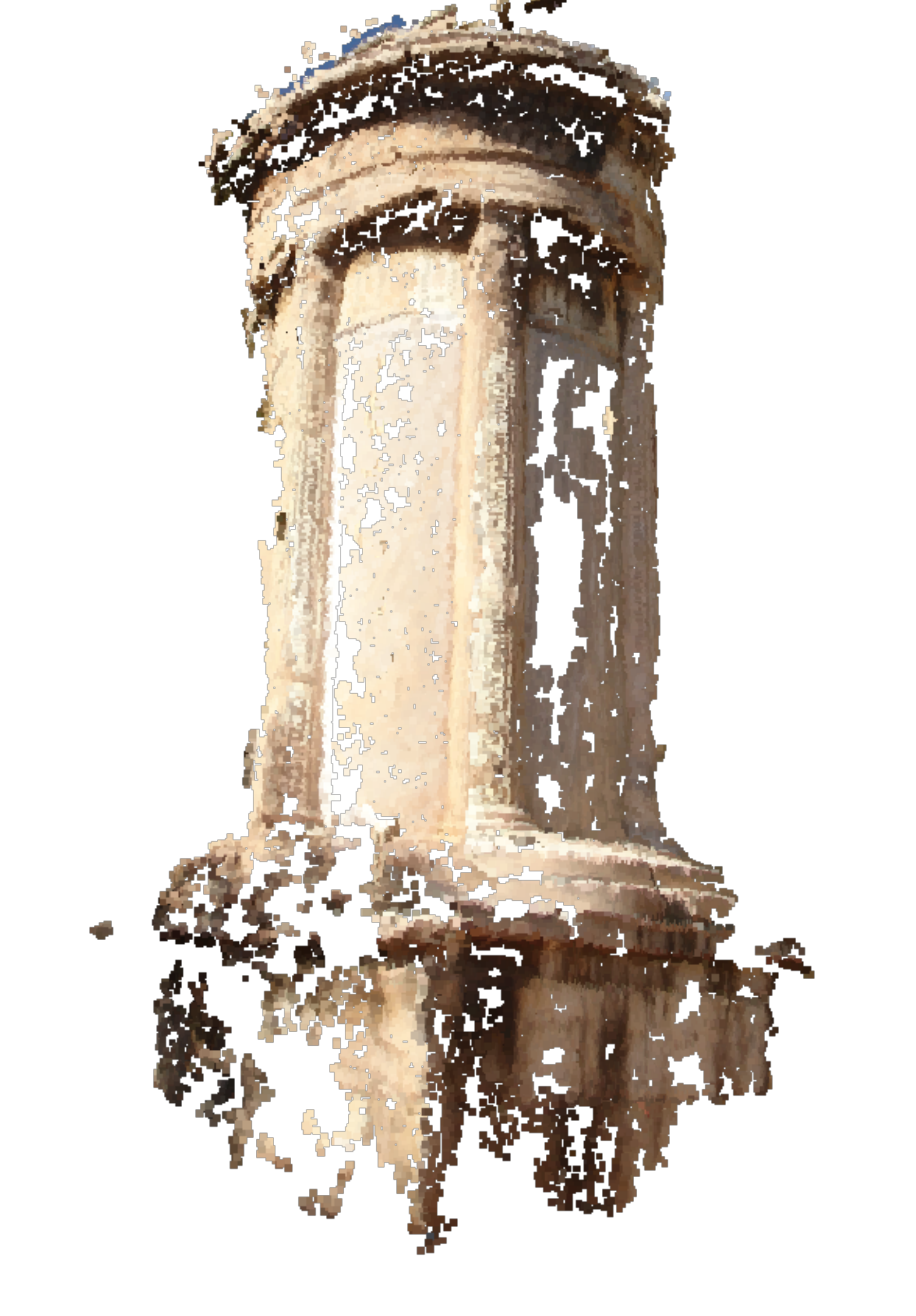}
\end{subfigure} \\
\begin{subfigure}{0.48\textwidth}
\includegraphics[width=\textwidth]{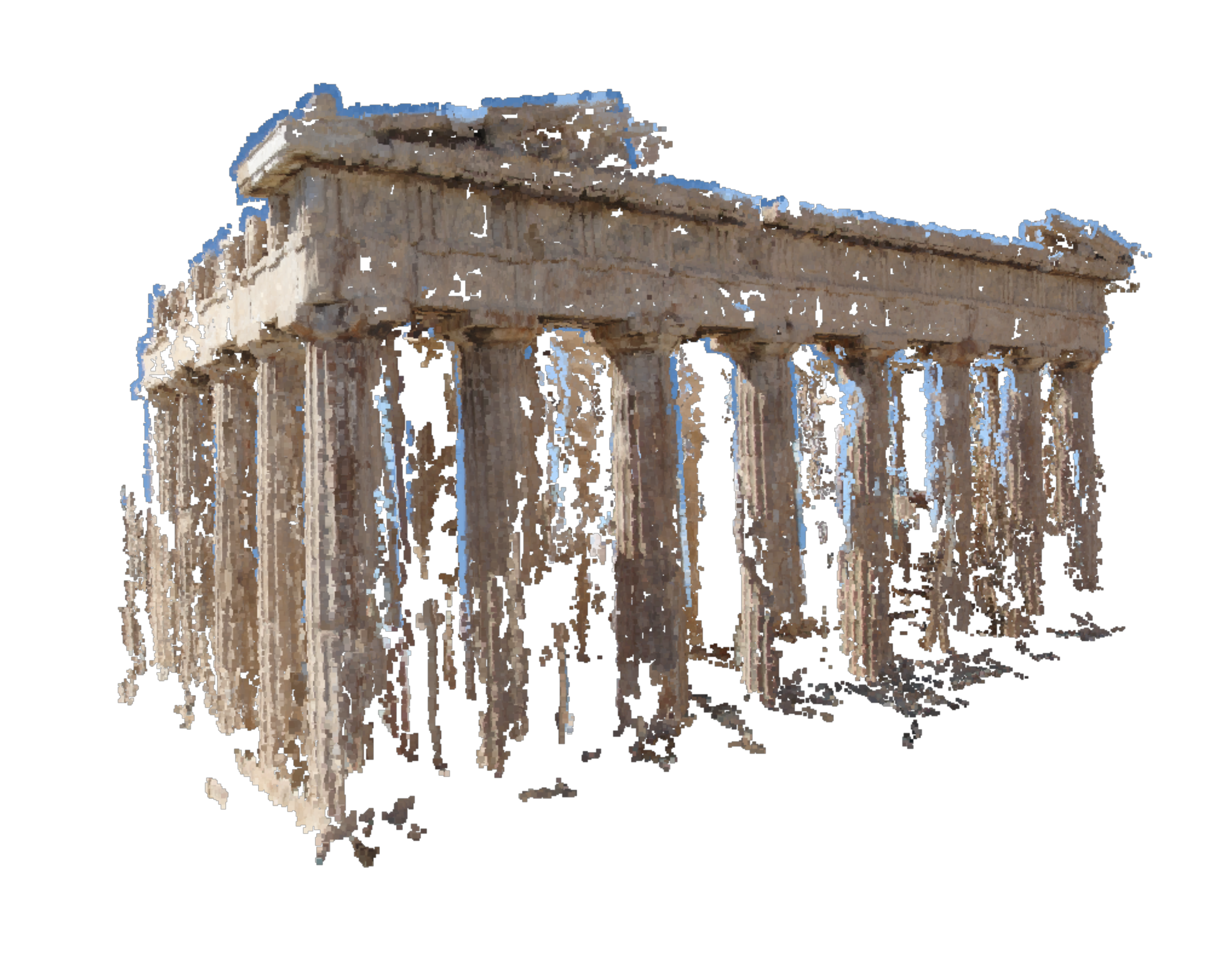}
\end{subfigure}~
\begin{subfigure}{0.48\textwidth}
\includegraphics[width=\textwidth]{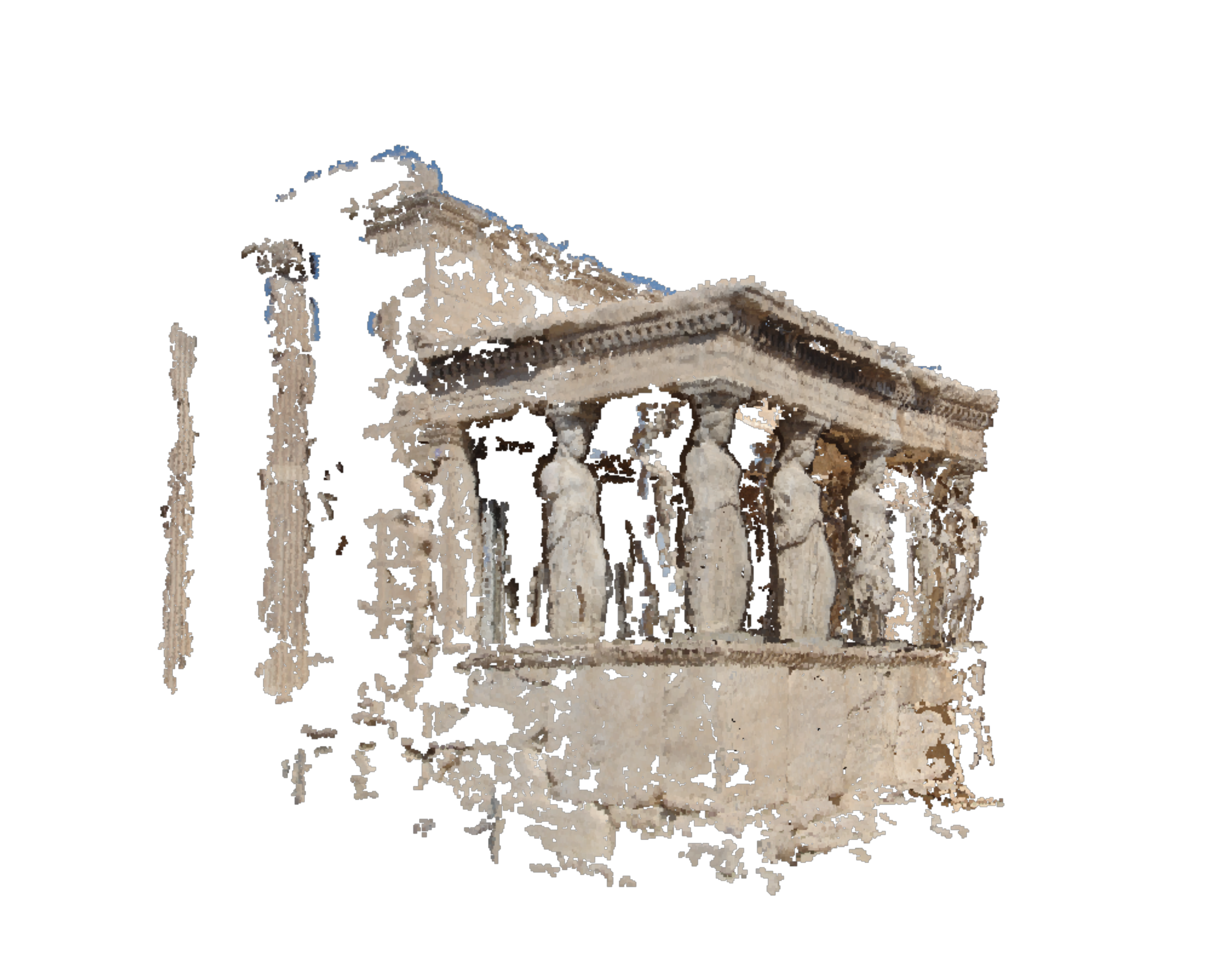}
\end{subfigure} 
\caption{3D reconstruction results, as dense point clouds. \textbf{Top Row:}  Datasets~\cite{strecha2008benchmarking}, Herz-Jesu-P25 (left) and Fountain-P11. \textbf{Middle Row:} Dataset castle-P19~\cite{strecha2008benchmarking} (left) and photo set of Monument of Lysicrates, Athens (right). \textbf{Bottom Row:} Photo sets of locations in Athens, Parthenon (left) and Karyatids (right). The scenes in Athens were photographed by the authors with a simple compact camera}
\label{fig:TwoReconstructions}
\end{figure*}

\begingroup
\hyphenpenalty 10000
\exhyphenpenalty 10000
\begin{table*}
\footnotesize
\begin{center}
\begin{tabularx}{0.99\textwidth}{R X X X  }
\midrule[1.5pt]
\textbf{Dataset}&$\Delta \mathbf{t} (^{\circ})$ &$\Delta R (^{\circ})$ &$\Delta f$ \\ \midrule[1.5pt] 
\textbf{castle-P30}&3.10&1.06&0.0389 \\ \midrule[0.25pt]
 \textbf{castle-P19}&7.35&4.17&0.0586 \\ \midrule[0.25pt]
\textbf{entry-P10}&4.62&4.67&0.2118\\ \midrule[0.25pt]
\textbf{Herz-Jesu-P8}&1.00&0.68&0.0266\\ \midrule[0.25pt]
\textbf{Herz-Jesu-P25}&0.41&0.31&0.0049 \\ \midrule[0.25pt]
\textbf{Fountain-P11}&0.44&0.41&0.0095 \\ 
\lasthline
\end{tabularx}
\caption{Mean errors in $\mathbf{t}, R, f$ estimation in the final reconstruction. Datasets from~\cite{strecha2008benchmarking}.}
\label{Tab:ReconstructionErrors}
\end{center}
\end{table*}
\endgroup

% Define Delta t , Delta R somewhere...... p155 thesis....
% \begin{equation*}
%\Delta t \equiv \angle \mathbf{t,t_{gt}}
%\end{equation*}
%\begin{equation*}
%\Delta R \equiv \angle RR_{gt}^{-1}
%\end{equation*} 

\section{Conclusions}
Using the DIAC, we developed a linear self-calibration and metric reconstruction method. Two theorems describe the relative configuration of the two recovered solutions and provide support to use the Cheirality condition for solution disambiguation. Comparisons to Kruppa equations and the 5P algorithm revealed that our method performs similarly to these standard approaches. Subsequently we show that the large number of $f,R$ estimates that are produced by our self-calibration and metric reconstruction method can be utilised through averaging methods, shifting our focus from choosing the best solution to finding, eg as in finding an optimised $f$ estimate prior to self-calibration, the best solution averaging method. We also developed a general method to verify point matches between images, which can be solved by reduction to LCS. The corresponding verification method can be used in any problem with image correspondences input. The verification method successfully rejected outliers in both architectural and  general scenes, with more success in the former category. All our methods were  integrated to a full multiple-view reconstruction pipeline to produce visually high-quality reconstructions on both standard datasets and image sets we shot using a conventional camera. Multi-view reconstructions were obtained combining camera pair reconstructions using rotation averaging algorithms and a novel approach to average focal length estimates.

\section*{References}
\bibliography{Bibliography}
\begin{appendices}
\section{Gaussian elimination in Self calibration and metric recontruction equations}
To simplify the expressions, we introduce the notation
\begin{equation*}
\mathbf{P_i^{\gamma}}\text{: row vector produced from i-th row of }  {[ \mathbf{a}]}_x F
\end{equation*}
and permute $\mathbf{x_0}$ elements with the permutation
\begin{equation*}
4 \leftrightarrow 6
\end{equation*}
We denote the permuted vector by $\mathbf{x}$ and the corresponding system matrix by $A_{pr}$. Using this notation, we write $A_{pr}$ as
\begin{equation}
\label{eq:Prapparent}
\begin{bmatrix}
\mathbf{r_1} & \phi^1_1 \mathbf{P_{2}^{\gamma}} & \psi_1 \\
\mathbf{r_2} & \phi_2^1 \mathbf{P_{2}^{\gamma}}+\phi_2^2 \mathbf{P_{3}^{\gamma}} & \psi_2 \\
\mathbf{r_3} & \phi_3^1 \mathbf{P_{1}^{\gamma}}+\phi_3^2 \mathbf{P_{3}^{\gamma}} & \psi_3 \\
\mathbf{r_4} & \phi_4^1 \mathbf{P_{1}^{\gamma}} & \psi_4 \\
\mathbf{r_5} & \phi_5^1 \mathbf{P_{1}^{\gamma}}+\phi_5^2 \mathbf{P_{2}^{\gamma}} & \psi_5 
\end{bmatrix} 
\end{equation}
where $\mathbf{r_i}$ are $1\times 3$ vectors and $\phi,\psi$ are appropriate constants of no special structure. 

We aim  to eliminate the elements in the rows $1-3$ and columns $4-6$ of $A_{pr}$, which we refer to as $A_n$, and then to apply regular Gaussian elimination. This is generally possible, owing to the structure of $A_n$ rows in~\eqref{eq:Prapparent},  which are linear combinations of $\mathbf{P_{i}^{\gamma}}$ vectors and, also, using the canonical projective reconstruction  allows us to substitute \begin{equation}
\label{eq:lindepProj}
\mathbf{P_3^{\gamma} }=d_1\mathbf{P_1^{\gamma} }+d_2\mathbf{P_2^{\gamma} }
\end{equation}
Thus, the elimination of $A_n$ elements is now straightforward by applying row-operations to matrix $A_{pr}$.  We then apply ordinary Gaussian elimination to reduce $A_{pr}$ to the form of~\eqref{eq:rref}.

\section{Geometric Relations between the two recovered solutions for  metric recontruction of a camera pair}

We  proceed with the proofs of  Theorems~\ref{th:mirror} and~\ref{th:viewdir}. 

\begin{ress}
\label{th:Cnull}
Let $P$ denote a projection matrix. The center of projection  $\mathbf{C_P}$ has no image, as it is projected to point $\mathbf{0}$. Equivalently, $\mathbf{C_P}=\begin{pmatrix} \mathbf{C^T}&1\end{pmatrix}^T$ is a right null-vector of $P$.
\end{ress}
\begin{ress}
\label{th:KRC}
Let $P$ denote a projection matrix.  $P$ can be decomposed as
\begin{equation*}
P=\begin{bmatrix} KR & -KR\mathbf{C}\end{bmatrix}
\end{equation*}
\end{ress}
Results~\ref{th:Cnull},~\ref{th:KRC} describe properties of the camera position $\mathbf{C}$. The following Result is concerned with the camera direction
\begin{ress}
\label{th:direction}
Assume a projection matrix
\begin{equation*}
 P=\begin{bmatrix} M & \mathbf{p} \end{bmatrix} 
\end{equation*}
Let the vector $\mathbf{m^{T}_3}$ denote the third row of $M$. Then the vector 

\begin{equation*}
\mathbf{v}=\det{(M)}\mathbf{m_3}
\end{equation*}
is in the direction of the principal axis (the viewing direction) of $P$ and is directed towards the front of the camera.
\end{ress}
The next two lemmas describe properties of metric reconstructions $P_{m2}^1, P_{m2}^2$ derived from Eq.~\eqref{eq:sols}
\begin{lemma}
\label{th:samet}
Let
\begin{equation*}
\begin{aligned}
P_{m2}^1&=\begin{bmatrix} K_2^1R^1 & \mathbf{a^1} \end{bmatrix}\\
P_{m2}^2&=\begin{bmatrix} K_2^2R^2 & \mathbf{a^2} \end{bmatrix}
\end{aligned}
\end{equation*}
be the projection matrices for camera $2$ derived from Eq.~\eqref{eq:sols}.

Then
\begin{equation*}
\mathbf{a^1=a^2 \triangleq a}
\end{equation*}
\end{lemma}
\begin{proof}
Considering:
\begin{enumerate}
\item The form of homography~\eqref{eq:Hform}
\item Eq.~\eqref{eq:PH}:$P_{m2}^i = P_{P2}H^i$ where $H^i$ denotes the homography obtained by substituting the i-th solution of Eq.~\eqref{eq:sols}
\end{enumerate}
the lemma is readily deduced
\end{proof}
\begin{lemma}
\label{th:differ1}
Let $P_{m2}^1,P_{m2}^2$ as in Lemma~\ref{th:samet}. We have:
\begin{equation}
\label{eq:differ1}
K_2^1R_2^1-K_2^2R_2^2=\mathbf{a}\mathbf{n^T}
\end{equation}
where $\mathbf{n}$ is an appropriate vector.
\end{lemma}
\begin{proof}
As in  proof of Lemma~\ref{th:samet}, by observing that $H^i$ for different $i$ values differ only in $\mathbf{v} \triangleq \mathbf{-p^T}K$
\end{proof}
We note that Lemma~\ref{th:differ1} is a general result, independent of~\eqref{eq:sols}. However, we omit this proof now.
\begin{lemma}
\label{th:plusminusdet}
For the reconstructions $P_{m2}^1=\begin{bmatrix}P_1 & \mathbf{a}\end{bmatrix}$, $P_{m2}^2=\begin{bmatrix}P_2 & \mathbf{a}\end{bmatrix}$ we have
\begin{equation}
\label{eq:plusminusdet}
\det{P_{1}}=\pm \det{P_{2}}
\end{equation} 
\end{lemma}
\begin{proof}
We formed $P_{m2}^1, P_{m2}^2$ from solutions of Eq.~\eqref{eq:omInter}, so the projection matrices have the same $\omega^{*}$. Using now Eq.~\eqref{eq:DACdef} we have:

\begin{align*}
\omega_{1}^{*}&=P_{m2}^1Q^{*}_{\infty}{P_{m2}^{1}}^T\\
&=P_1P_1^T\\
&=\omega_{2}^{*}\\ 
&=P_2P_2^T
\end{align*}
Using the following known relations
\begin{itemize}
\item $\det (A\cdot B)=\det A\cdot \det B$
\item $\det A=\det A^T$
\end{itemize}
we have
\begin{align*}
\det{P_1P_1^T}&=\det{P_2P_2^T}\Rightarrow \\
\det{P_1}^2&=\det{P_2}^2\Rightarrow\\
\det{P_1}&=\pm \det{P_2}
\end{align*}
\end{proof}
From this last proof, the next Lemma becomes apparent
\begin{lemma}
\label{th:eqKmatrices}
Concerning the reconstructions $P_{m2}^1,P_{m2}^2$ of Lemma~\ref{th:plusminusdet}, we have
\begin{equation*}
K_{m2}^1=K_{m2}^2
\end{equation*}
\end{lemma}
\begin{proof}
From the equality of $\omega_1^*$,$\omega_2^*$ and the diagonallity of the internal calibration matrices $K_{m2}^1,K_{m2}^2$ we prove the Lemma.
\end{proof}
We next refine Lemma~\ref{th:plusminusdet}, to lift the sign ambiguity in Eq.~\eqref{eq:plusminusdet}.
\begin{lemma}
\label{th:B1}
For the reconstructions $P_{m2}^1,P_{m2}^2$ we have
\begin{equation}
\label{eq:B1}
\mathbf{C_{m2}^1=C_{m2}^2\triangleq C} \iff \mathbf{n^TC}=0
\end{equation}
\end{lemma}
\begin{proof}
\begin{align}
&P_{m2}^1\begin{pmatrix} \mathbf{C} \\ 1 \end{pmatrix}=\mathbf{0} \notag \\
&\iff \mathbf{a}=-K_2R_2^1\mathbf{C} \label{eq:B1i}
\end{align}
Similarly for $P_{m2}^2$ 
\begin{align}
&P_{m2}^2\begin{pmatrix} \mathbf{C} \\ 1 \end{pmatrix}=\mathbf{0} \notag \\
&\iff K_2R_2^1\mathbf{C}+\mathbf{an^TC+a}=\mathbf{0}\notag \\
  &\iff K_2R_2^1\mathbf{C}+\mathbf{an^TC}-K_2R_2^1\mathbf{C}=\mathbf{0},\text{holds from Eq.~\eqref{eq:B1i}} \notag \\
&\iff \mathbf{an^TC}=\mathbf{0}\notag \\
&\iff \begin{pmatrix} \alpha _1 \mathbf{n^T} \\ \alpha _2 \mathbf{n^T} \\ \alpha _3 \mathbf{n^T} \end{pmatrix} \mathbf{C}=\mathbf{0} \notag \\
  &\iff \mathbf{n^TC}=0 
\end{align}
\end{proof}
\begin{lemma}
\label{th:B2}
For the reconstructions $P_{m2}^1,P_{m2}^2$ we have
\begin{equation}
\label{eq:B2}
\mathbf{C_{m1}^1=-C_{m2}^2\triangleq C} \iff \mathbf{n^TC}=-2
\end{equation}
\end{lemma}
\begin{proof}
As in the proof of Lemma~\ref{th:B1}
\begin{align}
&P_{m2}^1\begin{pmatrix} \mathbf{C} \\ 1 \end{pmatrix}=\mathbf{0} \notag \\
&\iff \mathbf{a}=-K_2R_2^1\mathbf{C} \label{eq:B2i}
\end{align}
Similarly, from matrix $P_{m2}^2$ we have
\begin{align}
&P_{m2}^2\begin{pmatrix} \mathbf{-C} \\ 1 \end{pmatrix}=\mathbf{0} \notag \\
&\iff -K_2R_2^1\mathbf{C} \mathbf{+an^TC+a}=\mathbf{0}\notag \\
&\iff \mathbf{a} \mathbf{+an^TC +a}=\mathbf{0},\text{from~\eqref{eq:B2i}} \notag \\
&\iff \mathbf{a}(2+\mathbf{n^TC})=\mathbf{0}\notag \\
&\iff \mathbf{n^TC}=-2,\text{provided $\mathbf{a\neq 0}$} 
\end{align}
\end{proof}
We complement each of Lemmas~\ref{th:B1},\ref{th:B2}, with Lemma~\ref{th:B3} and~\ref{th:B4} respectively. To prove the last two Lemmas, we use Eq.~\eqref{eq:rank1identity}
\begin{ress}
\label{th:rank1identity}
For each square, invertible matrix $X$, column-vector $\mathbf{c}$ and row-vector $\mathbf{r}$ we have
\begin{equation}
\label{eq:rank1identity}
\det{(X+\mathbf{cr})}=\det{X}\cdot \det{(1+\mathbf{r}X^{-1}\mathbf{c})}
\end{equation}
\end{ress}
\begin{lemma}
\label{th:B3}
For the reconstructions $P_{m2}^1,P_{m2}^2$ we have
\begin{equation}
\begin{aligned}
\label{eq:B3}
\det{P_1}&=\det{P_2} \iff \\
\mathbf{n^TC_1}&=0
\end{aligned}
\end{equation}
\end{lemma}

\begin{proof}
We use Result~\ref{th:rank1identity}, for which we note:
\begin{enumerate}
\item  $P_1$ is a full-rank matrix ( $\rank 3$) for every projection matrix. The exception, referred to in the literature as "camera at infinity", is out of our scope. Remember we are handling a metric reconstruction.
\item  $P_2$ can be expressed in terms of $P_1$, $\mathbf{n,a}$, thus permitting the application of Eq.~\eqref{eq:rank1identity} to determine $\det{P_2}$.
\end{enumerate}
Now applying the previous points, we have
\begin{align}
&\det{P_2}=\det{P_1}\notag \\
&\iff 1-\mathbf{n^T}{R_2^1}^TK_2^{-1}\mathbf{a}=1 \notag \\
&\iff \mathbf{n^T}{R_2^1}^TK_2^{-1}\mathbf{a}=0 \notag \\
&\iff -\mathbf{n^T}{R_2^1}^TK_2^{-1}K_2R_2^1\mathbf{C_1}=0\, \notag \\
&\text{, from~\eqref{eq:B1i}: $\mathbf{a}=-K_2R_2^1\mathbf{C_1}$} \notag \\
&\iff \mathbf{n^TC_1}=0\,\notag \\
&\text{, as $RR^T=I$ for rotation matrices $R$}  \notag
\end{align}
\end{proof}

\begin{lemma}
\label{th:B4}
For the reconstructions $P_{m2}^1,P_{m2}^2$ we have
\begin{equation}
\begin{aligned}
\det{P_1}&=-\det{P_2} \iff \\
\mathbf{n^TC_1}&=-2
\end{aligned}
\end{equation}
\end{lemma}
\begin{proof}
As in the proof of Lemma~\ref{th:B3},we have:
\begin{align}
&\det{P_2}=-\det{P_1}\notag \\
&\iff 1-\mathbf{n^T}{R_2^1}^TK_2^{-1}\mathbf{a}=-1 \notag \\
&\iff \mathbf{n^T}{R_2^1}^TK_2^{-1}\mathbf{a}=2 \notag \\
&\iff -\mathbf{n^T}{R_2^1}^TK_2^{-1}K_2R_2^1\mathbf{C_1}=2\,\notag \\
 &\text{, from Eq.~\eqref{eq:B1i}: $\mathbf{a}=-K_2R_2^1\mathbf{C_1}$} \notag \\
&\iff \mathbf{n^TC_1}=-2\,\notag \\
 &\text{,  as $RR^T=I$ for rotation matrices $R$} \notag
\end{align}
\end{proof}
Now, we show that the case of same-sign determinants ($\det P^1 = \det P^2$) produces a contradiction, and is so rejected. Regarding the notation in the following, we clarify that:
\begin{enumerate}
\item The projective reconstruction $P_{P2}$ is in the canonical representation form 
\begin{equation} \label{eq:caconicalCameraPair} \begin{bmatrix} {[ \mathbf{a}]}_x F &\mathbf{a}\end{bmatrix}  \end{equation}
with $ F^T\mathbf{a}=\mathbf{0} $
\item ${[\mathbf{a}]}_x$ denotes the anti-symmetric matrix defined to compute outer product with vector $\mathbf{a}$
\begin{equation*} {[ \mathbf{a}]}_x \mathbf{v} = \mathbf{a} \times \mathbf{v}  \end{equation*}
\item $\mathbf{e}$ denotes the right null vector of $F$, \begin{equation}\label{eq:nulle}F\mathbf{e}=\mathbf{0}\end{equation}
\end{enumerate}
\begin{lemma}
\label{th:skewAF}
Let
\begin{equation*}P_{P2}=\begin{bmatrix}A&\mathbf{a}\end{bmatrix}= \begin{bmatrix} {[ \mathbf{a}]}_x F &\mathbf{a}\end{bmatrix} \end{equation*}
denote the Projection matrix for camera $2$ in the projective reconstruction and $\mathbf{p}, \mathbf{p'}$ the solutions for $\mathbf{\pi_{\infty}}$ acquired from Eq.~\eqref{eq:sols}
\begin{equation*}
\begin{aligned}
\mathbf{p^T}=\begin{pmatrix} p_1 & p_2 & p_3 \end{pmatrix} \\
\mathbf{{p'}^T}=\begin{pmatrix} p_1' & p_2' & p_3' \end{pmatrix}
\end{aligned}
\end{equation*}
Then
\begin{equation}
\label{eq:skewAF}
\mathbf{p-p'}=\psi \mathbf{e_f}
\end{equation}
where
\begin{equation*}
\begin{aligned}
\mathbf{e_f}&=\begin{pmatrix} \sfrac{e_1}{f_1^2} \\ \sfrac{e_2}{f_1^2} \\ e_3  \end{pmatrix}
\end{aligned}
\end{equation*}
\end{lemma}

\begin{proof}
From Eq.~\eqref{eq:omInter} and because solutions~\eqref{eq:sols} share the same $f_1$ value, we have
\begin{equation*}
\resizebox{0.9\linewidth}{!}{\ensuremath{
\begin{aligned}
\omega_1^*=\omega_2^* &\iff \\
P_{P2}\begin{bmatrix} K_1K_1^T & -K_1K_1^T\mathbf{p} \\ -\mathbf{p^T}K_1K_1^T & \mathbf{p^T}K_1K_1^T\mathbf{p}\end{bmatrix}P_{P2}^T &=\\ P_{P2}\begin{bmatrix} K_1K_1^T & -K_1K_1^T\mathbf{p'} \\ -\mathbf{{p'}^T}K_1K_1^T & \mathbf{{p'}^T}K_1K_1^T\mathbf{p'}\end{bmatrix}P_{P2}^T &\iff \\
AK_1K_1^TA^T - AK_1K_1^T\mathbf{pa^T} -\mathbf{ap^T}K_1K_1A^T +\mathbf{ap^T}K_1K_1^T\mathbf{pa^T}&=\\
 AK_1K_1^TA^T - AK_1K_1^T\mathbf{{p'}a^T} -\mathbf{a{p'}^T}K_1K_1A^T +\mathbf{a{p'}^T}K_1K_1^T\mathbf{p'a^T}
\end{aligned}}}
\end{equation*}
From Eqs.~\eqref{eq:sols}, \eqref{eq:rref3} we have
\begin{equation*} f_1^2p_1^2+f_2^2p_2^2+p_3^2 = f_1^2{p'}_1^2+f_2^2{p'}_2^2+{p'}_3^2=b_3 \end{equation*}
and so
\begin{equation}
\label{eq:ToEliminate1}
\resizebox{0.85\linewidth}{!}{\ensuremath{
(f_1^2p_1^2+f_1^2p_2^2+p_3^2)\mathbf{aa^T}=\mathbf{a{p}^T}K_1K_1^T\mathbf{pa^T}=\mathbf{a{p'}^T}K_1K_1^T\mathbf{p'a^T}}}
\end{equation}
By eliminating the common terms (Eq.~\eqref{eq:ToEliminate1} and $AK_1K_1A^T$) we continue the computations and arrive at
\begin{align}
AK_1K_1^T(\mathbf{(p-p')}\mathbf{a^T})+(\mathbf{a(p^T-{p'}^T))}K_1K_1^TA^T&=0 \iff \notag\\
Q+Q^T&=0\label{eq:pssQequalsigns}
\end{align}
In Eq.~\eqref{eq:pssQequalsigns} we defined
\begin{equation*}
Q\triangleq AK_1K_1^T\left(\mathbf{\left(p-p'\right)}\mathbf{a^T}\right)
\end{equation*}
We write $Q$ as
\begin{align}
Q&= AK_1K_1^T(\mathbf{(p-p')}\mathbf{a^T}) \notag \\
&=AK_1K_1^T\begin{pmatrix} (p_1-p_1')\mathbf{a^T} \\ (p_2-p_2')\mathbf{a^T}\\ (p_3-p_3')\mathbf{a^T} \end{pmatrix} \notag \\
&=A\begin{pmatrix} f_1^2(p_1-p_1')\mathbf{a^T} \\ f_1^2(p_2-p_2')\mathbf{a^T}\\ (p_3-p_3')\mathbf{a^T} \end{pmatrix} \notag \\
&=\begin{pmatrix} \mathbf{A^1}^T \\ \mathbf{A^2}^T \\ \mathbf{A^3}^T \end{pmatrix} \begin{pmatrix} \Delta _f \alpha _1 & \Delta _f \alpha _2 & \Delta _f \alpha _3  \end{pmatrix} \label{eq:MatrixSSQ}
\end{align}
where in Eq.~\eqref{eq:MatrixSSQ} we defined
\begin{equation*}
\Delta _f \triangleq  \begin{pmatrix} f_1^2(p_1-p_1') \\ f_1^2(p_2-p_2')\\ (p_3-p_3') \end{pmatrix}
\end{equation*}

 From Eq.~\eqref{eq:pssQequalsigns}, matrix $Q$ is anti-symmetric and so has a zero diagonal. Imposing the last condition  on expression~\eqref{eq:MatrixSSQ}, we extract the following relations
\begin{align}
\mathbf{A^1}^T\Delta _f &= 0 \label{eq:SSQ1} \\
\mathbf{A^2}^T\Delta _f &= 0 \label{eq:SSQ2} \\
\mathbf{A^3}^T\Delta _f &= 0 \label{eq:SSQ3}
\end{align} 
We  substitute $A$ in Eqs.~\eqref{eq:SSQ1},\eqref{eq:SSQ2},\eqref{eq:SSQ3}, using the canonical representation assumption
\begin{equation*}
P_{P2}=\begin{bmatrix}A&\mathbf{a}\end{bmatrix}= \begin{bmatrix} {[ \mathbf{a}]}_x F &\mathbf{a}\end{bmatrix}
\end{equation*}
and write the three resulting equations in matrix form to get
\begin{equation}
\label{eq:SSQnull}
 {[ \mathbf{a}]}_xF\Delta_f  = \mathbf{0}
\end{equation}
From Eq.~\eqref{eq:SSQnull} and because  ${[ \mathbf{a}]}_xF$ has the null vector $\mathbf{e}$, we get
\begin{equation}
\label{eq:DeltaSSQ}
\Delta_f=\psi \mathbf{e}
\end{equation}
where $\psi$ is a constant.

Now, from Eq.~\eqref{eq:DeltaSSQ}, with simple manipulations we obtain:
\begin{equation*}
\begin{pmatrix} p_1-p_1' \\ p_2-p_2' \\p_3-p_3' \end{pmatrix} = \psi \mathbf{e_f} 
\end{equation*}
\end{proof}
\begin{lemma}
\label{th:finaldetsign}
With the assumptions and notation of  Lemma \ref{th:skewAF}, we have
\begin{equation*}
\det{P_1}=-\det{P_2}
\end{equation*} 
\end{lemma} 

\begin{proof}
We assume that 
\begin{equation*}
\det{P_1} = \det{P_2}
\end{equation*}
and produce a contradiction.

From Lemma~\ref{th:B3}, we get Eq.~\eqref{eq:B3} and equivalently require that:
\begin{equation}
\label{eq:innerProdEquiv}
\mathbf{n^TC_1}=0
\end{equation}
To specify $\mathbf{n}$ in Eq.~\eqref{eq:innerProdEquiv}, we use
\begin{enumerate}
\item The definition of $\mathbf{n}$ in Eq.~\eqref{eq:differ1}
\item The relation between $P_{P2},P_{M2},H$ (Eqs.~\eqref{eq:PH},\eqref{eq:Hform}) and the notation for $P$ matrix of Lemma~\ref{th:skewAF}
\end{enumerate}
and have
\begin{align*}
P_1&=AK_1 -\mathbf{ap^T}K_1\\
P_2&=AK_1 -\mathbf{ap'^T}K_1\\
\iff P_2&=P_1 +\mathbf{a (p-p')^T}K_1\triangleq P_1-\mathbf{an^T}
\end{align*}
Now, we can rewrite Eq.~\eqref{eq:innerProdEquiv} as
\begin{equation}
\label{eq:innerProdEquiv2}
\mathbf{(p-p')^TK_1C_1}=0
\end{equation}
We next have 
\begin{equation*}
\begin{aligned}
P_{M2}^1\begin{pmatrix} \mathbf{C^1} \\1 \end{pmatrix}& =0\iff \\
P_{P2}H^1\begin{pmatrix} \mathbf{C^1} \\1 \end{pmatrix}&=0 \iff \\
P_{P2}\begin{pmatrix} K_1\mathbf{C^1} \\-\mathbf{p^T}K_1\mathbf{C^1}+1 \end{pmatrix}&=0 
\end{aligned}
\end{equation*}
From the assumption that $P_{P2}$ is in the canonical form (Eq.~\eqref{eq:caconicalCameraPair}), it has a null vector (Eq.~\eqref{eq:nulle}) that is written as
\begin{equation*}
\begin{pmatrix} \mathbf{e} \\ 0\end{pmatrix}
\end{equation*}
So we have:
\begin{align}
P_{P2}\begin{pmatrix} K_1\mathbf{C^1} \\\mathbf{-p^T}K_1\mathbf{C^1}+1 \end{pmatrix}&=0 \iff \notag \\
K_1\mathbf{C^1}&=\psi \mathbf{e} \label{eq:k1cnull}\text{\,,where $\psi$ is a constant}\\
\mathbf{-p^T}K_1\mathbf{C^1}+1&=0 \notag %\label{eq:pk1cnull}
\end{align}
From Lemma~\ref{th:skewAF} (Eq.~\eqref{eq:skewAF}) and the previous Eqs.~\eqref{eq:innerProdEquiv2}, \eqref{eq:k1cnull} we get:
\begin{align}
\mathbf{e_F^Te}&=0  \iff \notag \\
\sfrac{e_1^2}{f_1^2}+\sfrac{e_2^2}{f_1^2} + e_3^2&=0 \label{eq:NOTgonnahappen}
\end{align}
Since Eq.~\eqref{eq:NOTgonnahappen} has no solutions (~$\mathbf{e}\neq \mathbf{0}$~), we produced a contradiction. 

Thus, from Lemma~\ref{th:plusminusdet}, we have proved that
\begin{equation*}
\det{P_1}=-\det{P_2}
\end{equation*} 
\end{proof}
From the preceding Lemmas, we can now readily obtain Theorem~\ref{th:mirror}
\begin{proof}[Proof of Theorem~\ref{th:mirror}]
From Lemma~\ref{th:finaldetsign} 
\begin{equation*}
\det{P1}=-\det{P2}
\end{equation*}
From Lemmas~\ref{th:B2},\ref{th:B4} we obtain the equivalent relation
\begin{equation*}
\mathbf{C^1_{m2}=-C^2_{m2}}
\end{equation*}
\end{proof}

Next, we prove Theorem~\ref{th:viewdir}. To avoid a lengthy proof, we settle the coplanarity of $\mathbf{v_{m2}^{1,2},C_{m2}^{1,2}}$ with Lemma~\ref{th:lemmaplane}, which follows the main proof.

 Let us first summarize some notation
\begin{enumerate}
\item We denote $ \mathbf{v_{m2}^1},\mathbf{{v_{m2}^2}}$ the vectors that point along the viewing directions  of cameras $P_{m2}^1,P_{m2}^2$ respectively
\item For  $P_{m2}^1$ we assume
\begin{align*}
\det{P_1}&>0 \\
\mathbf{C}&\triangleq \mathbf{C^1_{m2}} 
\end{align*}
\end{enumerate}
\begin{proof}[Proof of Theorem~\ref{th:viewdir}]
From Results~\ref{th:KRC}, \ref{th:direction}, Lemma~\ref{th:samet}, Theorem~\ref{th:mirror} we have  for $P_{m2}^1$:
\begin{align}
K_2R^1\mathbf{C}&= -\mathbf{a}\iff \notag\\
\begin{pmatrix} f_2\mathbf{R_1}^T \\ f_2\mathbf{R_2}^T \\ \mathbf{R_3}^T \end{pmatrix} \mathbf{C} &= -\mathbf{a} \iff \notag \\
\mathbf{R_3}^T\mathbf{C}&=-a_3 \label{eq:angleCV1}
\end{align}   
We have $\det{P^1}=\det{K_2R_1}>0$ and so
\begin{equation*}
\mathbf{v_{2m}^1}=\mathbf{R_3}
\end{equation*}
Consequently, from Eq.~\eqref{eq:angleCV1}, we have
\begin{equation}
\mathbf{v_{2m}^1}^T\mathbf{C}=\| \mathbf{v_{2m}^1} \|\|\mathbf{C} \| \cos\angle \mathbf{C}, \mathbf{v_{2m}^1}=-a_3 \label{eq:angleCV2}
\end{equation}
In Eq.~\eqref{eq:angleCV2}, 
\begin{equation*}
\|\mathbf{R_3^T}\|=1\\
\text{, since $R^1$ is orthogonal as a rotation matrix}
\end{equation*}
We can normalize $\mathbf{C}$ to unitary by satisfying the condition
\begin{equation*}
\|K_2^{-1}\mathbf{a}\|=1
\end{equation*}
since rotations leave vectors' measure unchanged.\\
We can now write Eq.~\eqref{eq:angleCV2} as
\begin{equation*}
 \cos\angle \mathbf{C}, \mathbf{v_{2m}^1}=-a_3
\end{equation*}
Similarly, using 
\begin{align*}
\mathbf{C_{m2}^2}&=\mathbf{-C} \\
\det{P^2}&<0 \\
\mathbf{a^1}&=\mathbf{a^2}
\end{align*}
we have
\begin{equation*}
 \cos\angle \mathbf{C}, \mathbf{v_{m2}^2}=-a_3
\end{equation*}
and the remaining relations required for the proof:
\begin{align*}
 \cos\angle \mathbf{C^2}, \mathbf{v_{m2}^1}&=a_3\\
\cos\angle \mathbf{C^2}, \mathbf{v_{m2}^2}&=a_3 \\
\angle \mathbf{C}, \mathbf{v_{m2}^1}+\angle \mathbf{C^2}, \mathbf{v_{m2}^1}&=180^{\circ}\\
\angle \mathbf{C}, \mathbf{v_{m2}^2}+\angle \mathbf{C^2}, \mathbf{v_{m2}^2}&=180^{\circ}\\
\end{align*}
To complete the proof, we show that $\mathbf{v_{m2}^{1,2},C_{m2}^{1,2}}$ are coplanar. We provide a constructive proof in Lemma~\ref{th:lemmaplane}.
\end{proof}
\begin{lemma}
\label{th:lemmaplane}
There exist rotation matrices $R_x$, $R_{perm}$ so that
\begin{align*}
R_xR_{perm}R^1\mathbf{v_{m2}^1} &= \begin{pmatrix} 1 & 0 & 0 \end{pmatrix} ^T \\
R_xR_{perm}R^1\mathbf{v_{m2}^2} &= \begin{pmatrix} x & y & 0 \end{pmatrix} ^T
\end{align*}
\end{lemma}
\begin{proof}
  In this proof, we apply to the 3D space similarity transforms, that do not alter angles. The aim is to transform the space so that the resulting coordinate system simplifies the relations of the entities we examine.
  
  We visualize this process as placing and orienting a "virtual" camera, so that the camera primary plane is the plane on which $\mathbf{v_{m2}^{1,2},C_{m2}^{1,2}}$ lie. We first do some hypotheses,  without loss of generality, to simplify the notation in the proof:
\begin{itemize}
\item Let $P_{m2}^1$ denote the correct representation of $P_{m2}$ and $P_{m2}^2$ the erroneous one
\item Let \begin{equation*}
\sign ( \det P_1)>0
\end{equation*}
so that we can simplify the expression for camera viewing direction
\end{itemize}
 We apply to space the rotations
\begin{equation*}
%\label{eq:thetaphitrig1}
R_xR_{perm}R^1
\end{equation*}
where
\begin{align*}
R^1:&\text{ rotation matrix of $P^1_{m2}$}\\
R_{perm}:& \text{ rotation to transpose $x_1,x_3$}\\
\phantom{:}& \text{of a vector: $\begin{pmatrix} x_1 & x_2& x_3 \end{pmatrix}^T$}\\
R_x:&\text{rotation to place $\mathbf{C^1}$ in the desired plane }
\end{align*}
Applying $R^1$, using orthogonality of $R^1$ and Result~\ref{th:direction}, we have for the viewing direction of camera $2$
\begin{equation*}
R^1\mathbf{v_{m2}^1}=\begin{pmatrix} 0& 0& 1 \end{pmatrix}^T
\end{equation*}
We then apply $R_{perm}$, to help with the visualization of this proof
\begin{equation*}
R_{perm}=R_y(90^{\circ})=\begin{bmatrix} 0 & 0 & 1 \\ 0 & 1 & 0 \\ -1& 0 & 0 \end{bmatrix} 
\end{equation*}
We define $R_x$, a rotation around $x-$axis, to place $\mathbf{C^1}$ on $z-$plane and at the same time leave $\mathbf{v_{m2}^1}$ unchanged. We have
\begin{equation*}
R_x=\begin{bmatrix} 1 & 0 & 0 \\ 0 & \cos \theta_x  & -\sin \theta_x \\ 0& \sin \theta_x &  \cos \theta_x \end{bmatrix}
\end{equation*}
We transformed  $\mathbf{C^1}$ to:  
\begin{align*}
K_2R^1\mathbf{C^1}&=-\mathbf{a}\iff \\
R^1\mathbf{C^1}&= -K_2^{-1}\mathbf{a} \iff \\
R_{perm}R^1\mathbf{C^1}&=-R_{perm}\left(-K_2^{-1}\right)\mathbf{a}= \begin{pmatrix} -a_3 \\ -f_2a_2 \\ f_2a_1 \end{pmatrix}
\end{align*}
Then, applying $R_x$  we have:
\begin{align*}
R_xR_{perm}R^1\mathbf{C^1}&=R_x \begin{pmatrix} -a_3 \\ -f_2a_2 \\ f_2a_1 \end{pmatrix}\\
&=  \begin{pmatrix} x_1 \\ x_2 \\ -f_2a_2\sin \theta_x +f_2a_1\cos \theta_x \end{pmatrix}\\
\end{align*}
To satisfy the condition ($x_3=0$) for $\mathbf{C^1}$, we get for $R_x$
\begin{align*}
-f_2a_2\sin \theta_x +f_2a_1\cos \theta_x&=0 \iff \\
\tan \theta_x &= \sfrac{a_1}{a_2}
\end{align*}
Now, we show that the $R_x$ we specified previously, also places  $\mathbf{v_{m2}^2}$ on the $z-$plane of the virtual camera. 
Let 
\begin{equation*}
\mathbf{m_j^i}:\text{$j-$row of $R^i$}
\end{equation*}
We have
\begin{small}
\begin{align*}
 R_{perm}R^1\mathbf{v_{m2}^2}&=R_{perm}R^1(\mathbf{-m_3^2}^T) \\
&= \begin{pmatrix}\mathbf{m_3^1}(\mathbf{-m_3^2}^T) & \mathbf{m_2^1}(\mathbf{-m_3^2}^T) & -\mathbf{m_1^1}(\mathbf{-m_3^2}^T) \end{pmatrix}^T\\
&\triangleq \begin{pmatrix} x_1 & x_2 & x_3 \end{pmatrix}^T   
\end{align*}
\end{small}
where we used that $\det{P^2}<0$. Now, it suffices to show
\begin{small}
\begin{align*}
R_x\begin{pmatrix} x_1 & x_2 & x_3 \end{pmatrix}^T &= \begin{pmatrix} x_1' & x_2' & 0 \end{pmatrix}\iff \\
\sin \theta _x \mathbf{m_2^1}\mathbf{m_3^2}^T &= \cos \theta _x \mathbf{m_1^1}\mathbf{m_3^2}^T \iff \\
 a_1 (\mathbf{m_2^2} +f_2^{-1}a_2\mathbf{n^T})\mathbf{m_3^2}^T &= a_2 (\mathbf{m_1^2} +f_2^{-1}a_1\mathbf{n^T})\mathbf{m_3^2}^T \iff \\
a_1f_2^{-1}a_2\mathbf{n^T}\mathbf{m_3^2}^T&= a_2f_2^{-1}a_1\mathbf{n^T}\mathbf{m_3^2}^T
\end{align*} 
\end{small}
where we used Lemma~\ref{th:differ1}, diagonal form of $K_2$,  the orthogonality of rotation matrices and that $\mathbf{a_1}=\mathbf{a_2}$. 

Thus, we proved that $\mathbf{v_{m2}^1,v_{m2}^2,C^1,C^2}$ lie in the plane $z=0$  of the transformed world coordinate system, which is the primary plane of the "virtual" camera.
\end{proof}

\end{appendices}
\end{document}